\documentclass[twoside]{article}

\usepackage[accepted]{aistats2022}
\usepackage[round]{natbib}

\usepackage[utf8]{inputenc} 
\usepackage[T1]{fontenc}    
\usepackage{hyperref}       
\usepackage{url}            
\usepackage{booktabs}       
\usepackage{amsfonts}       
\usepackage{nicefrac}       
\usepackage{microtype}      
\usepackage{xcolor}         

\usepackage{comment}
\usepackage{algorithmic}
\usepackage{algorithm}
\usepackage{color}
\usepackage{mathrsfs}
\usepackage{float}
\usepackage{amsmath} 
\usepackage{amssymb}
\usepackage{amsfonts}
\usepackage{amsthm}
\usepackage{stmaryrd}
\usepackage{stackrel}
\usepackage{graphicx}
\usepackage{hyperref}
\usepackage{dsfont}
\usepackage{tikz}
\usetikzlibrary{decorations.pathreplacing}
\usepackage{framed}
\usepackage{appendix}
\usepackage{enumerate}
\usepackage{placeins}
\usepackage{authblk}
\usepackage{tikz-layers}
\usepackage{multirow}

\numberwithin{equation}{section}
\theoremstyle{plain}
\newtheorem{theorem}{Theorem}

\newtheorem{lemme}{Lemma}

\newtheorem*{assumption}{}

\newcommand{\Dn}{\mathscr{D}_{n}}

\newcommand{\bX}{\textbf{X}}

\newcommand{\bTheta}{\boldsymbol{\Theta}}

\newcommand{\bx}{\textbf{x}}

\renewcommand{\P}{\mathds{P}}
\newcommand{\R}{\mathds{R}}
\newcommand{\E}{\mathbb{E}}
\newcommand{\V}{\mathbb{V}}

\newcommand{\bU}{\mathcal{U}_{n,K}}
\newcommand{\bSh}{\mathbf{\hat{Sh}}}
\newcommand{\INDSTATE}[1][1]{\STATE\hspace{#1\algorithmicindent}}

\begin{document}

%

%

\twocolumn[

\aistatstitle{SHAFF: Fast and consistent \underline{SHA}pley e\underline{F}fect estimates via random \underline{F}orests}

\aistatsauthor{Clément Bénard$^{1,2}$ \And Gérard Biau$^{2}$ \And Sébastien Da Veiga$^{1}$ \And Erwan Scornet$^{3}$}

\aistatsaddress{ \\[-0.5em] $^{1}$ Safran Tech, Digital Sciences \& Technologies, 78114 Magny-Les-Hameaux, France \\ 
$^{2}$Sorbonne Université, CNRS, LPSM, 75005 Paris, France \\
$^{3}$Ecole Polytechnique, IP Paris, CMAP, 91128 Palaiseau, France } 

\runningauthor{C. Bénard, G. Biau, S. Da Veiga, E. Scornet}

]

\begin{abstract}
    Interpretability of learning algorithms is crucial for applications involving critical decisions, and variable importance is one of the main interpretation tools. 
    Shapley effects are now widely used to interpret both tree ensembles and neural networks, as they can efficiently handle dependence and interactions in the data, as opposed to most other variable importance measures. 
    However, estimating Shapley effects is a challenging task, because of the computational complexity and the conditional expectation estimates. Accordingly, existing Shapley algorithms have flaws: a costly running time, or a bias when input variables are dependent. Therefore, we introduce \textbf{SHAFF}, \textbf{SHA}pley e\textbf{F}fects via random \textbf{F}orests, a fast and accurate Shapley effect estimate, even when input variables are dependent. We show \textbf{SHAFF} efficiency through both a theoretical analysis of its consistency, and the practical performance improvements over competitors with extensive experiments. An implementation of \textbf{SHAFF} in \texttt{C++} and \texttt{R} is available online.
\end{abstract}

\section{Introduction}

State-of-the-art learning algorithms are often qualified as black boxes because of the high number of operations required to compute predictions. This complexity prevents us from grasping how inputs are combined to generate the output, which is a strong limitation for many applications, especially those with critical decisions at stake---healthcare is a typical example. For this reason, interpretability of machine learning has become a topic of strong interest in the past few years. One of the main tools to interpret learning algorithms is variable importance, which enables to identify and rank the influential features of the problem.
Recently, Shapley effects have been widely accepted as a very efficient variable importance measure since they can equitably handle interactions and dependence within input variables \citep{owen2014sobol, vstrumbelj2014explaining, iooss2017shapley, lundberg2017unified}.
Shapley values were originally defined in economics and game theory \citep{shapley1953value} to solve the problem of attributing the value produced by a joint team to its individual members. The main idea is to measure the difference of produced value between a subset of the team and the same subteam with an additional member. For a given member, this difference is averaged over all possible subteams and gives his Shapley value. Recently, \citet{owen2014sobol} adapted Shapley values to variable importance in machine learning, where an input variable plays the role of a member of the team, and the produced value is the explained output variance. In this context, Shapley values are now called Shapley effects, and are extensively used to interpret both tree ensembles and neural networks. Next, \citet{lundberg2017unified} also introduced SHAP values to adapt Shapley effects to local importance measures, which break down the contribution of each variable for a given prediction. We focus on Shapley effects throughout the article, but our approach can be easily adpated to SHAP values as they share the same challenges.

The objective of variable importance is essentially to perform variable selection. More precisely, it is possible to identify two final aims \citep{genuer2010variable}: (i) find a small number of variables with a maximized accuracy, or (ii) detect and rank all influential variables to focus on for interpretation and further exploration with domain experts. The following example illustrates that different strategies should be used depending on the targeted objective: if two influential variables are strongly correlated, one must be discarded for objective (i), while the two of them must be kept in the second case. Indeed, if two variables convey the same statistical information, only one should be selected if the goal is to maximize the predictive accuracy with a small number of variables, i.e., objective (i). On the other hand, these two variables may be acquired differently and represent distinct physical quantities. Therefore, they may have different interpretations for domain experts, and both should be kept for objective (ii). 
Shapley effects are a relevant measure of variable importance for objective (ii), because they equitably allocate contributions due to interactions and dependence across all input variables.

The main obstacle to estimate Shapley effects is the computational complexity. The first step is to use a learning algorithm to generalize the relation between the inputs and the output. Most existing Shapley algorithms are agnostic to the learning model. \citet{lundberg2018consistent} open an interesting route by restricting their algorithm to tree ensembles, in order to develop fast greedy heuristics, specific to trees. Unfortunately, as mentioned by \citet{aas2019explaining}, the algorithm is biased when input variables are dependent. In the present contribution, we propose a Shapley algorithm tailored for random forests, well known for their good behavior
on high-dimensional or noisy data, and their robustness. Using the specific structure of random forests, we develop \textbf{SHAFF}, a fast and accurate estimate of Shapley effects.

\paragraph{Shapley effects.}
To formalize Shapley effects, we introduce a standard regression setting with an input vector $\bX = (X^{(1)}, \hdots, X^{(p)}) \in \R^p$, and an output $Y \in \R$. We denote by $\smash{\bX^{(U)}}$ the subvector with only the components in $U \subset \{1,\hdots,p\}$.
Formally, the Shapley effect of the $j$-th variable is defined by
\begin{align*}
	Sh^{\star}(X^{(j)}) = & \sum_{U \subset \{1,\hdots,p\}\setminus\{j\}} \frac{1}{p} {p - 1 \choose |U|}^{-1} \frac{1}{\mathbb{V}[Y]} \\ & \times \big( \V[\E[Y|\bX^{(U\cup\{j\})}]] - \V[\E[Y|\bX^{(U)}]] \big).
\end{align*}
In other words, the Shapley effect of $X^{(j)}$ is the additional output explained variance when $j$ is added to a subset $U \subset \{1,\hdots,p\}$, averaged over all possible subsets. The variance difference is averaged for a given size of $U$ through the combinatorial weight, and then the average is taken over all $U$ sizes through the term $1/p$.
Observe that the sum has $2^{p-1}$ terms, and each of them requires to estimate $\smash{\V[\E[Y|\bX^{(U)}]]}$, which is computationally costly. Overall, two obstacles arise to estimate Shapley effects:
\begin{enumerate}
    \item the computational complexity is exponential with the dimension $p$;
    \item $\V[\E[Y|\bX^{(U)}]]$ requires a fast and accurate estimate for all variable subsets $U \subset \{1,\hdots,p\}$.
\end{enumerate}
In the literature, efficient strategies have been developed to handle these two issues. They all have drawbacks: they are either fast but with a limited accuracy, or accurate but computationally costly. We will see how \textbf{SHAFF} considerably improves this trade-off.

\paragraph{Related work.}
The computational issue of Shapley algorithms---$1.$ above---is solved using Monte-Carlo methods \citep{song2016shapley, lundberg2017unified, covert2020understanding, williamson2020efficient, covert2020improving}. In the case of tree ensembles, specific heuristics based on the tree structure enable to simplify the algorithm complexity \citep{lundberg2018consistent}.

For the second issue of conditional expectation estimates---$2.$ above, two main approaches exist: train one model for each selected subset of variables (accurate but computationally costly) \citep{williamson2020efficient}, or train a single model once with all input variables and use greedy heuristics to derive the conditional expectations (fast but limited accuracy). In the latter case, existing algorithms estimate the conditional expectations with a quite strong bias when input variables are dependent. More precisely, \citet[kernelSHAP]{lundberg2017unified}, \citet[SAGE]{covert2020understanding}, and \citet{covert2020improving} simply replace the conditional expectations by the marginal distributions, \citet{lundberg2018consistent} use a greedy heuristic specific to tree ensembles, and \citet{broto2020variance} leverage $k$-nearest neighbors to approximate sampling from the conditional distributions.
Besides, efficient algorithms exist when it is possible to draw samples from the conditional distributions of the inputs \citep{song2016shapley, aas2019explaining, broto2020variance}. However, we only have access to a finite sample in practice, and the input dimension $p$ can be large, which implies that estimating the conditional distributions of the inputs is a very difficult task. This last type of methods is therefore not really appropriate in our setting---see Table \ref{table_algo} for a summary of the existing Shapley algorithms.
\begin{table*}
    \setlength{\tabcolsep}{2pt}
	\centering
	\begin{tabular}{|c | c | c | c | c |}
		\hline \hline
        Reference & Model & \begin{tabular}{c} Local or \\ global \end{tabular} & \begin{tabular}{c} Subset \\ sampling\end{tabular} & \begin{tabular}{c} Conditional \\ expectations \end{tabular} \\ 
		\hline \hline
		\citet{song2016shapley} & all & global & \begin{tabular}{c} permutation \end{tabular} & \begin{tabular}{c} kown conditional \\ distributions \end{tabular} \\	\hline
		\citet{lundberg2017unified} & all & local & Monte-Carlo & marginals \\ \hline
		\citet{lundberg2018consistent} & tree ensembles & local  & greedy heuristic &  greedy heuristic \\ \hline
		\citet{aas2019explaining} & all & local & Monte-Carlo & \begin{tabular}{c} kown conditional \\ distributions \end{tabular} \\ \hline
		\citet{covert2020understanding} & all & global & Monte-Carlo &  marginals \\ \hline
		\citet{broto2020variance} & all & global & brute force & k-nearest neighbors \\ \hline
		\citet{williamson2020efficient} & all & global & Monte-Carlo & retrain model \\ \hline
		\citet{covert2020improving} & all & local & Monte-Carlo & marginals \\
		\hline
		\textbf{SHAFF} & random forests & global & importance sampling & projected forests \\
		\hline \hline
	\end{tabular}
	\caption{State-of-the-art of Shapley Algorithms.}
	\label{table_algo}
\end{table*}

As mentioned above, several of the presented methods provide local importance measures for specific prediction points, called SHAP values \citep{lundberg2017unified, lundberg2018consistent, covert2020improving}. Their final objective differs from ours, since we are interested in global estimates. However, SHAP values share the same challenges as Shapley effects: the computational complexity and the conditional expectation estimates, and our approach can therefore be adapted to SHAP values.
Let us also mention that several recent articles discuss Shapley values in the causality framework \citep{NEURIPS2020_0d770c49, NEURIPS2020_32e54441, janzing2020feature, pmlr-v130-wang21b}. These works have a high potential since causality is quite often the ultimate goal when one is looking for interpretations. 
However, causality methods require strong prior knowledge and assumptions about the studied system, and can therefore be difficult to deploy in specific applications.  In these cases, we argue that it is preferable to use Shapley effects---also known as global conditional Shapley values---to detect and rank influential variables, as a starting point to deepen the analysis with domain experts.

\paragraph{Outline.}
We leverage random forests to develop \textbf{SHAFF}, a fast and accurate Shapley effect estimate. Such remarkable performance is reached by combining two new features. Firstly, we improve the Monte-Carlo approach by using importance sampling to focus on the most relevant subsets of variables identified by the forest. Secondly, we develop a projected random forest algorithm to compute fast and accurate estimates of the conditional expectations for any variable subset. The algorithm details are provided in Section \ref{sec_algo}.
Next, we prove the consistency of \textbf{SHAFF} in Section \ref{sec_theory}. To our knowledge, \textbf{SHAFF} is the first Shapley effect estimate, which is both computationally fast and consistent in a general setting.
In Section \ref{sec_xp}, several experiments show the practical improvement of our method over state-of-the-art algorithms.

\section{SHAFF Algorithm} \label{sec_algo}

\paragraph{Existing approach.}
\textbf{SHAFF} builds on two Shapley algorithms: \citet[kernelSHAP]{lundberg2017unified} and \citet{williamson2020efficient}. From these approaches, we can deduce the following general three-step procedure to estimate Shapley effects.
First, a set $\bU$ of $K$ variable subsets $U \subset \{1,\hdots,p\}$ is randomly drawn. Next, an estimate $\hat{v}_n(U)$ of $\smash{\V[\E[Y|\bX^{(U)}]]}$ is computed for all selected $U$ from an available sample $\Dn = \{(\bX_1, Y_1), \hdots, (\bX_n, Y_n) \}$ of $n$ independent random variables distributed as $(\bX, Y)$.
Finally, Shapley effects are defined as the least square solution of a weighted linear regression problem. If $I(U)$ is the binary vector of dimension $p$ where the $j$-th component takes the value $1$ if $j \in U$ and $0$ otherwise, Shapley effect estimates are the minimum in $\beta$ of the following cost function:
\begin{align*}
    \ell_{n}(\beta) = \frac{1}{K} \sum_{U \in \bU} w(U) (\hat{v}_{n}(U) - \beta^T I(U))^2,
\end{align*}
where the weights $w(U)$ are given by
\begin{align*}
w(U) = \frac{p - 1}{{p \choose |U|} |U| (p - |U|)},
\end{align*}
and the coefficient vector $\beta$ is constrained to have its components sum to $\hat{v}_{n}(\{1,\hdots,p\})$.

\paragraph{Algorithm overview.}
\textbf{SHAFF} introduces two new critical features to estimate Shapley effects efficiently, using an initial random forest. Firstly, we apply importance sampling to select variable subsets $U \subset \{1,\hdots, p\}$, based on the variables frequently selected in the forest splits. This favors the sampling of subsets $U$ containing influential and interacting variables. Secondly, for each selected subset $U$, the variance of the conditional expectation is estimated with the projected forest algorithm described below, which is both a fast and consistent approach. We will see that these features considerably reduce the computational cost and the estimate error. 
To summarize, once an initial random forest is fit, \textbf{SHAFF} proceeds in three steps:
\begin{enumerate}
    \item sample many subsets $U$, typically a few hundreds, based on their occurrence frequency in the random forest (Subsection \ref{subsec_imp_samp});
    \item estimate $\V[\E[Y|\bX^{(U)}]]$ with the projected forest algorithm for all selected $U$ and their complementary sets $\{1,\hdots, p\} \setminus U$ (Subsection \ref{subsec_PRF});
    \item solve a weighted linear regression problem to recover Shapley effects (Subsection \ref{subsec_linreg}).
\end{enumerate}

\paragraph{Initial random forest.}
Prior to \textbf{SHAFF}, a random forest is fit with the training sample $\Dn$ to generalize the relation between the inputs $\bX$ and the output $Y$. A large number $M$ of CART trees are averaged to form the final forest estimate $m_{M,n}(\bx, \bTheta_{M})$, where $\bx$ is a new query point, and each tree is randomized by a component of $\bTheta_{M} = (\Theta_1,\hdots,\Theta_{\ell},\hdots,\Theta_M)$. Each $\smash{\Theta_{\ell}}$ is used to bootstrap the data prior to the $\ell$-th tree growing, and to randomly select \texttt{mtry} variables to optimize the split at each node. \texttt{mtry} is a parameter of the forest, and its efficient default value is $p/3$. In the sequel, we will need the forest parameter \texttt{min\_node\_size}, which is the minimum number of observations in a terminal cell of a tree, as well as the out-of-bag (OOB) sample of the $\ell$-th tree: the observations which are left aside in the bootstrap sampling prior to the construction of tree $\ell$. Given this initial random forest, we can now detail the main three steps of \textbf{SHAFF}.

\subsection{Importance Sampling} \label{subsec_imp_samp}

The Shapley effect formula for a given variable $X^{(j)}$ sums terms over all subsets of variables $U \subset \{1,\hdots,p\}\setminus\{j\}$, which makes $2^{p-1}$ terms, an intractable problem in most cases.
\textbf{SHAFF} uses importance sampling to draw a reasonable number of subsets $U$, typically a few hundreds, while preserving a high accuracy of the Shapley estimates. We take advantage of the initial random forest to define an importance measure for each variable subset $U$, used as weights for the importance sampling distribution.

\paragraph{Variable subset importance.}
In a tree construction, the best split is selected at each node among \texttt{mtry} input variables. Therefore, as highlighted by Proposition $1$ in \citet{scornet2015consistency}, the forest naturally splits on influential variables. \textbf{SHAFF} leverages this idea to define an importance measure for all variable subsets $U \subset \{1,\hdots,p\}$ as the probability that a given $U$ occurs in a path of a tree of the forest. 
Empirically, this means that we count the occurrence frequency of $U$ in the paths of the $M$ trees of the forest, and denote it by $\hat{p}_{M,n}(U)$.
Such approach is inspired by \citet{basu2018iterative} and \citet{benard2021sirus}. 
This principle is illustrated with the following simple example in dimension $p = 10$. Let us consider a tree, where the root node splits on variable $\smash{X^{(5)}}$, the left child node splits on variable $\smash{X^{(3)}}$, and the subsequent left child node at the third tree level, on variable $\smash{X^{(2)}}$. Thus, the path that leads to the extreme left node at the fourth level uses the following index sequence of splitting variables: $\{5, 3, 2\}$. All in all, the following variable subsets are included in this tree path: $U = \{5\}$, $U = \{3, 5\}$, and $U = \{2, 3, 5\}$. This subset extraction favors variables at the top of trees, which are more influential on $Y$, and is computationally fast.
Then, \textbf{SHAFF} runs through the forest to count the number of times each subset $U$ occurs in the forest paths, and computes the associated frequency $\hat{p}_{M,n}(U)$. If a subset $U$ does not occur in the forest, we have $\hat{p}_{M,n}(U) = 0$. Notice that the computational complexity of this step is linear, i.e., $O(Mn)$, since the number of operations is proportional to the number of nodes in the forest and there are about $n$ nodes in each fully grown tree.

\paragraph{Paired importance sampling.}
The occurrence frequencies $\hat{p}_{M,n}(U)$ defined above are scaled to sum to $1$, and then define a discrete distribution for the set of all subsets of variables $U \subset \{1,\hdots,p\}$, excluding the full and empty sets. By construction, this distribution is skewed towards the subsets $U$ containing influential variables and interactions, and is used for the importance sampling.
Finally, \textbf{SHAFF} draws a number $K$ of subsets $U$ with respect to this discrete distribution, where $K$ is a hyperparameter of the algorithm. We define $\bU$ the random set of the selected variable subsets $U$. 
For all $U \in \bU$, \textbf{SHAFF} also includes the complementary set $\{1,\hdots, p\} \setminus U$ in $\bU$, as \citet{covert2020improving} show that this ``paired sampling'' improves the final Shapley estimate accuracy.
Clearly, the computational complexity and the accuracy of the algorithm increase with $K$.
The next step of \textbf{SHAFF} is to efficiently estimate $\smash{\V[\E[Y|\bX^{(U)}]]}$ for all drawn $U \in \bU$.

\subsection{Projected Random Forests} \label{subsec_PRF}

In order to estimate $\V[\E[Y|\bX^{(U)}]]$ for the selected variable subsets $U \in \bU$, most existing methods use greedy algorithms. However, such estimates are not accurate in moderate or large dimensions when input variables are dependent \citep{aas2019explaining, sundararajan2020many}. Another approach is to train a new model for each subset $U$, but this is computationally costly \citep{williamson2020efficient}. To solve this issue, we design the projected random forest algorithm (PRF), to obtain a fast and accurate estimate of $\smash{\V[\E[Y|\bX^{(U)}]]/\V[Y]}$ for any variable subset $U \subset \{1,\hdots,p\}$.

\paragraph{PRF principle.} PRF takes as inputs the initial forest and a given subset $U$. The general principle is to project the partition of each tree of the forest on the subspace spanned by the variables in $U$, as illustrated in Figure \ref{fig_proj_CART}. Then the training data is spread across this new tree partitions, and the cell outputs are recomputed by averaging the output $Y_i$ of the observations falling in each new cell, as in the original forest. The projection enables to eliminate the variables not contained in $U$ from the tree predictions, and thus to estimate $\smash{\E[Y|\bX^{(U)}]}$ instead of $\smash{\E[Y|\bX]}$. Finally, the predictions for the out-of-bag samples are computed with the projected tree estimates, and averaged across all trees. The obtained predictions are used to estimate the targeted normalized variance $\smash{\V[\E[Y|\bX^{(U)}]]/\V[Y]}$, denoted by $\hat{v}_{M,n}(U)$.
More formally, we let $\smash{m_{M,n}^{(U, OOB)}(\bX_i^{(U)}, \bTheta_M)}$ be the out-of-bag PRF estimate for observation $i$ and subset $U$, and take
\begin{align*}
    \hat{v}_{M,n}(U) = 1 - \frac{1}{n \hat{\sigma}_Y} \sum_{i=1}^{n} \big(Y_i - m_{M,n}^{(U, OOB)}(\bX_i^{(U)}, \bTheta_M)\big)^2,
\end{align*}
where $\hat{\sigma}_Y$ is the standard estimate of $\V[Y]$.
\begin{figure*}
\centering
\begin{tikzpicture}[scale = 0.9]

\draw[->] (-0,0) -- (5.5,0);
\draw (5.5,0) node[below] {$X^{(1)}$};
\draw [->] (0,-0) -- (0,5);
\draw (0,5) node[above] {$X^{(2)}$};
\draw [color = blue, line width = 0.4 mm] (3.5,2) -- (3.5,5);
\draw [color = blue, line width = 0.4 mm] (1,2) -- (1,3.5);
\draw [color = blue, line width = 0.4 mm] (2.5,2) -- (2.5,-0);
\draw [color = blue, line width = 0.4 mm] (5.5,2) -- (-0,2);
\draw [color = blue, line width = 0.4 mm] (-0,3.5) -- (3.5,3.5);
\draw [color = blue, line width = 0.4 mm] (4.8,2) -- (4.8,-0);
\filldraw (1.5,2.5) circle[radius=2pt];
\draw (1.6,2.5) node[right] {$\bX$};
\filldraw[color = cyan] (0.4, 1.5) circle[radius=2pt];
\filldraw[color = cyan] (2.2, 0.5) circle[radius=2pt];
\filldraw[color = cyan] (1.2, 4.3) circle[radius=2pt];
\filldraw[color = green] (1.8, 3) circle[radius=2pt];
\filldraw[color = green] (3.3, 2.4) circle[radius=2pt];
\filldraw[color = green] (2.8, 3.2) circle[radius=2pt];
\filldraw[color = cyan] (5, 2.7) circle[radius=2pt];
\filldraw[color = cyan] (4.3, 0.8) circle[radius=2pt];

\draw[->, line width = 0.4 mm, color = red] (8,0) -- (13.5,0);
\draw (13.5,0) node[below] {$X^{(1)}$};
\draw [->, dashed] (8,-0) -- (8,5);
\draw (8,5) node[above] {$X^{(2)}$};
\draw [color = blue, line width = 0.4 mm, dashed] (11.5,2) -- (11.5,5);
\draw [color = blue, line width = 0.4 mm, dashed] (9,2) -- (9,3.5);
\draw [color = blue, line width = 0.4 mm, dashed] (10.5,2) -- (10.5,-0);
\draw [color = blue, line width = 0.4 mm, dashed] (13.5,2) -- (8,2);
\draw [color = blue, line width = 0.4 mm, dashed] (8,3.5) -- (11.5,3.5);
\draw [color = blue, line width = 0.4 mm, dashed] (12.8,2) -- (12.8,0);
\draw [color = red, line width = 0.2 mm] (11.5,-0) -- (11.5,5);
\draw [color = red, line width = 0.2 mm] (9,-0) -- (9,5);
\draw [color = red, line width = 0.2 mm] (10.5,5) -- (10.5,-0);
\draw [color = red, line width = 0.2 mm] (12.8,5) -- (12.8,-0);
\filldraw (9.5,2.5) circle[radius=2pt];
\draw (9.6,2.5) node[right] {$\bX$};
\draw [line width = 0.2 mm, dashed] (9.5,2.5) -- (9.5,0);
\filldraw (9.5,0) circle[radius=2pt];
\draw (9.5,0) node[above right] {$\bX^{(U)}$};
\filldraw[color = cyan, opacity = 0.3] (8.4, 1.5) circle[radius=2pt];
\filldraw[color = green, opacity = 0.3] (10.2, 0.5) circle[radius=2pt];
\filldraw[color = green, opacity = 0.3] (9.2, 4.3) circle[radius=2pt];
\filldraw[color = green, opacity = 0.3] (9.8, 3) circle[radius=2pt];
\filldraw[color = cyan, opacity = 0.3] (11.3, 2.4) circle[radius=2pt];
\filldraw[color = cyan, opacity = 0.3] (10.8, 3.2) circle[radius=2pt];
\filldraw[color = cyan, opacity = 0.3] (13, 2.7) circle[radius=2pt];
\filldraw[color = cyan, opacity = 0.3] (12.3, 0.8) circle[radius=2pt];
\filldraw[color = cyan] (8.4, 0) circle[radius=2pt];
\filldraw[color = green] (10.2, 0) circle[radius=2pt];
\filldraw[color = green] (9.2, 0) circle[radius=2pt];
\filldraw[color = green] (9.8, 0) circle[radius=2pt];
\filldraw[color = cyan] (11.3, 0) circle[radius=2pt];
\filldraw[color = cyan] (10.8, 0) circle[radius=2pt];
\filldraw[color = cyan] (13, 0) circle[radius=2pt];
\filldraw[color = cyan] (12.3, 0) circle[radius=2pt];

\end{tikzpicture}
\caption{Example of the partition of $[0,1]^2$ by a random CART tree (left side) projected on the subspace spanned by $\bX^{(U)} = X^{(1)}$ (right side). Here, $p = 2$ and $U = \{1\}$.}
\label{fig_proj_CART}
\end{figure*}
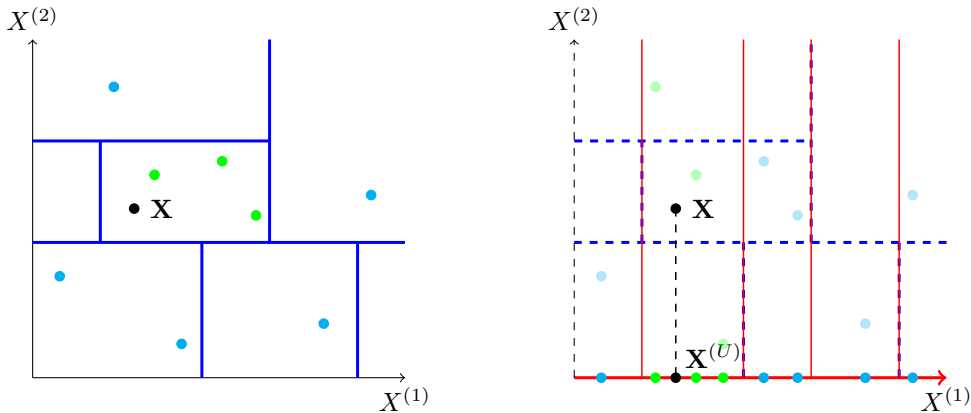

\paragraph{PRF algorithm.} The critical feature of PRF is the algorithmic trick to compute the projected partition efficiently, leaving the initial tree structures untouched. Indeed, a naive computation of the projected partitions from the cell edges is computationally very costly, as soon as the dimension increases. Instead, we simply drop observations down the initial trees, ignoring splits which use a variable outside of $U$. This enables to recover the projected partitions with an efficient computational complexity. To explain this mechanism in details, we focus on a given tree of the initial forest. Thus, the training observations are dropped down the tree, and when a split involving a variable outside of $U$ is met, data points are sent both to the left and right children nodes. Consequently, each observation falls in multiple terminal leaves of the tree. We drop the new query point $\smash{\bX^{(U)}}$ down the tree, following the same procedure, and retrieve the set of terminal leaves where $\smash{\bX^{(U)}}$ falls. Next, we collect the training observations which belong to every terminal leaf of this collection, in other words, we intersect the collection of leaves where $\smash{\bX^{(U)}}$ falls. Finally, we average the outputs $Y_i$ of the selected training points to generate the tree prediction for $\smash{\bX^{(U)}}$. Notice that such set of selected observations can be empty if $\smash{\bX^{(U)}}$ belongs to a large collection of terminal leaves. To avoid this issue, PRF uses the following strategy. Recall that a partition of the input space is associated to each tree level, and consequently, a projected tree partition can also be defined at each tree level. Thus, when $\smash{\bX^{(U)}}$ is dropped down the tree, it is stopped before reaching a tree level where it falls in an empty cell of the associated projected partition. Overall, this mechanism is equivalent to the projection of the tree partition on the subspace span by $\smash{X^{(U)}}$, because all splits on variables $\smash{X^{(j)}}$ with $\smash{j \notin U}$ are ignored, and the resulting overlapping cells are intersected---see Figure \ref{fig_proj_CART}.

\paragraph{PRF computational complexity.}
An efficient implementation of the PRF algorithm is provided by Algorithm \ref{algo_PRF} in Appendix \ref{appendix_B}. The computational complexity of PRF for all $U \in \bU$ does not depend on the dimension $p$, is linear with $M$, $K$, and quasi-linear with $n$: $O(\smash{MKn\log(n))}$. PRF is therefore faster than growing $K$ random forests from scratch, one for each subset $U$, which has an averaged complexity of $O(\smash{MKpn\log^2(n))}$ \citep{louppe2014understanding}. The computational gain of \textbf{SHAFF} can be considerable in high dimension, since the complexity of all competitors depends on $p$---see Table \ref{table_complexity} and Appendix \ref{appendix_B} for a detailed complexity analysis. As we will see in Section \ref{sec_xp}, $K=500$ is an efficient default setting for \textbf{SHAFF}, and thus, SAGE is faster for $p < 500$.
Notice that the PRF algorithm is close in spirit to a component of the Sobol-MDA \citep{benard2021mda}, used to measure the loss of output explained variance when an input variable is removed from a random forest.
In particular, a naive adaptation leads to a quadratic complexity with respect to the sample size $n$, whereas our PRF algorithm has a quasi-linear complexity, which makes it operational.
Finally, the last step of \textbf{SHAFF} is to take advantage of the estimated $\hat{v}_{M,n}(U)$ for $U \in \bU$ to recover Shapley effects.
\begin{table}
    \setlength{\tabcolsep}{2pt}
	\centering
	\begin{tabular}{|c | c |}
		\hline \hline
        Algorithm & Complexity \\ 
		\hline \hline
		\citet{covert2020understanding} & $O(Mpn\log(n))$ \\ \hline
		\citet{broto2020variance} & $O(n\log(n)p2^p)$ \\ \hline
		\citet{williamson2020efficient} & $O(\smash{Mpn^2\log^2(n))}$ \\ \hline
		\textbf{SHAFF} & $O(MKn\log(n))$ \\
		\hline \hline
	\end{tabular}
	\caption{Computational Complexity.}
	\label{table_complexity}
\end{table}

\subsection{Shapley Effect Estimates} \label{subsec_linreg}
The importance sampling introduces the corrective terms $\hat{p}_{M,n}(U)$ in the final loss function. Thus, \textbf{SHAFF} estimates $\smash{\bSh_{M,n} = (\hat{Sh}_{M,n}(X^{(1)}), \hdots, \hat{Sh}_{M,n}(X^{(p)}))}$ as the minimum in $\beta$ of the following cost function
\begin{align*}
    \ell_{M,n}(\beta) = \frac{1}{K} \sum_{U \in \bU} \frac{w(U)}{\hat{p}_{M,n}(U)} (\hat{v}_{M,n}(U) - \beta^T I(U))^2,
\end{align*}
where the sum of the components of $\beta$ is constrained to be the proportion of output explained variance of the initial forest, fit with all input variables. Finally, this can be written in the following compact form:
\begin{align*}
    \bSh_{M,n} = \underset{\beta \in [0,1]^p}{\textrm{argmin}} & \quad \ell_{M,n}(\beta) \\ 
    \textrm{s.t. }&  ||\beta||_1 = \hat{v}_{M,n}(\{1,\hdots,p\}).
\end{align*}

\section{SHAFF Consistency} \label{sec_theory}

We prove in this section that \textbf{SHAFF} is consistent, in the sense that the estimated value is arbitrarily close to the ground truth theoretical Shapley effect, provided that the sample size is large enough. To our knowledge, we provide the first Shapley algorithm which requires to fit only a single initial model and is consistent in the general case. We insist that our result is valid even when input variables exhibit strong dependences.
The consistency of \textbf{SHAFF} holds under the following mild and standard assumption on the data distribution:
\begin{assumption}[A1] \label{A1}
The response $Y \in \R$ follows
\begin{align*}
    Y = m(\bX) + \varepsilon,
\end{align*}
where $\bX = (X^{(1)}, \hdots, X^{(p)}) \in [0,1]^p$ admits a density over $[0,1]^p$ bounded from above and below by strictly positive constants, $m$ is continuous, and the noise $\varepsilon$ is sub-Gaussian, independent of $\bX$, and centered.
\end{assumption}

To alleviate the mathematical analysis, we slightly modify the standard Breiman random forests: the bootstrap sampling is replaced by a subsampling without replacement of $a_n$ observations, as it is usually done in the theoretical analysis of random forests \citep{scornet2015consistency, mentchquantifying2016}. Additionally, we follow \citet{wager2018estimation} with an additional small modification of the forest algorithm, which is sufficient to ensure its consistency. Firstly, a node split is constrained to generate child nodes with at least a small fraction $\gamma > 0$ of the parent node observations. 
Secondly, the split selection is slightly randomized: at each tree node, the number \texttt{mtry} of candidate variables drawn to optimize the split is set to $\texttt{mtry} = 1$ with a small probability $\delta > 0$. Otherwise, with probability $1 - \delta$, the default value of \texttt{mtry} is used. 
It is stressed that these last modifications are mild, since $\gamma$ and $\delta$ can be chosen arbitrarily small.

Finally, we introduce the following two assumptions on the asymptotic regime of the algorithm parameters. Assumption (A2) enforces that the tree partitions are not too complex with respect to the sample size $n$. On the other hand, Assumption (A3) states that the number of trees and the number of sampled variable subsets $U$ grow with $n$. This ensures that all possible variable subsets have a positive probability to be drawn, which is required for the convergence of our algorithm based on importance sampling. 
\begin{assumption}[A2]
The asymptotic regime of $a_n$, the size of the subsampling without replacement, and the number of terminal leaves $t_n$ are such that $a_n \leq n-2$, $a_n/n < 1 - \kappa$ for a fixed $\kappa > 0$, $\lim \limits_{n \to \infty} a_n = \infty$, $\lim \limits_{n \to \infty} t_n = \infty$, and $\lim \limits_{n \to \infty} 2^{t_n} \frac{(\log(a_n))^9}{a_n} = 0$.
\end{assumption}
\begin{assumption}[A3]
    The number of Monte-Carlo sampling $K_n$ and the number of trees $M_n$ grow with $n$, such that $M_n \longrightarrow \infty$ and $n.M_n/K_n \longrightarrow 0$.
\end{assumption}

We also let the theoretical Shapley effect vector be $\mathbf{Sh}^{\star} = (Sh^{\star}(X^{(1)}), \hdots, Sh^{\star}(X^{(p)}))$ to formalize our main result.
\begin{theorem} \label{thm_shap}
    If Assumptions (A1), (A2), and (A3) are satisfied, then \textbf{SHAFF} is consistent, that is
    \begin{align*}
        \bSh_{M_n,n} \overset{p}{\longrightarrow} \mathbf{Sh}^{\star}.
    \end{align*}
\end{theorem}

The proof of Theorem \ref{thm_shap} relies on the following three lemmas.
Lemma $1$ states that all variable subsets $U$ have a positive probability to be drawn asymptotically, which ensures that the importance sampling approach can converge.
Lemma $2$ establishes the consistency of the projected forest estimate, and the proof uses arguments from \citet{gyorfi2006distribution} to control both the approximation and estimation errors.
Lemma $3$ combines the two previous lemmas to state the convergence of the loss function of the weigthed regression problem solved to recover Shapley effect estimates. The proofs of these lemmas are gathered in Appendix \ref{appendix_C}.
\begin{lemme} \label{lemma_proba}
    If Assumptions (A2) and (A3) are satisfied, for all $U \subset \{1, \hdots, p\}$, we have
    \begin{align*}
        \P\big( \hat{p}_{M_n,n}(U) > 0 \big) \longrightarrow 1.    
    \end{align*}
\end{lemme}
\begin{lemme} \label{lemma_proj}
    If Assumptions (A1) and (A2) are satisfied, the PRF is consistent, that is, for all $M \in \mathbb{N}^{\star}$ and $U \subset \{1,\hdots,p\}$, 
    \begin{align*}
        \hat{v}_{M,n}(U) \overset{p}{\longrightarrow} \V[\E[Y|\bX^{(U)}]]/\V[Y] \overset{\rm def}{=} v^{\star}(U).
    \end{align*}
\end{lemme}
We let $Z$ be a discrete random variable taking values in the set of all subsets of $\{1, \hdots, p\}$, excluding the full and empty sets. The discrete distribution of $Z$ is given by the weights $w(U)$ (the weights are scaled to sum to 1).
\begin{lemme} \label{lemma_loss}
    If Assumptions (A1), (A2), and (A3) are satisfied, we have
    \begin{align*}
        \ell_{M,n}(\beta) \overset{p}{\longrightarrow} \E[(v^{\star}(Z) - \beta^T I(Z))^2] \overset{\rm def}{=} \ell^{\star}(\beta).
    \end{align*}
\end{lemme}
\begin{proof}[Proof of Theorem $1$]
    We assume that Assumptions (A1), (A2), and (A3) are satisfied.
    Since $\ell^{\star}$ is convex and $\beta$ belongs to the compact set $[0,1]^p$, the pointwise convergence of Lemma \ref{lemma_loss} gives the uniform convergence
    \begin{align*}
        \sup_{\beta \in [0,1]^p} |\ell_{M,n}(\beta) - \ell^{\star}(\beta)| \overset{p}{\longrightarrow} 0.
    \end{align*}
    Additionally, since $\ell^{\star}$ is a quadratic convex function and the constraint domain $[0,1]^p$ is convex, $\ell^{\star}$ has a unique minimum. According to Theorem $2$ from \citet{lundberg2017unified}, this unique minimum is $\mathbf{Sh}^{\star}$.
	Finally, since the minimum of $\ell^{\star}$ is unique and $\ell_{M,n}$ uniformly converges to $\ell^{\star}$, we apply Theorem $5.7$ from \citet[page 45]{van2000asymptotic} to conclude that
	\begin{align*}
	    \bSh_{M,n} \overset{p}{\longrightarrow} \mathbf{Sh}^{\star}.
	\end{align*}
\end{proof}

\section{Experiments} \label{sec_xp}

We run three batches of experiments to show the improvements of \textbf{SHAFF} over the main competitors \citet{broto2020variance}, \citet{williamson2020efficient}, and \citet[SAGE]{covert2020understanding}. Experiment $1$ is a simple linear case with a redundant variable, while Experiment $2$ is a non-linear example with high-order interactions. In both cases, existing Shapley algorithms exhibit a bias which significantly modifies the accurate variable ranking, as opposed to \textbf{SHAFF}. Next, we combine the new features of \textbf{SHAFF} with existing algorithms to break down the performance improvements due to the importance sampling and the projected forest.
Finally, Experiment $3$ shows the good empirical behavior of \textbf{SHAFF} for categorical variables, although we focus on continuous variables throughout the paper for the sake of clarity.

\paragraph{Experiment settings.}
Our implementation of \textbf{SHAFF} is based on \texttt{ranger}, a fast random forest software written in \texttt{C++} and \texttt{R} from \citet{wright2017ranger}, and is available at \url{https://gitlab.com/drti/shaff}.
We implemented \citet{williamson2020efficient} from scratch, as it only requires to sample variable subsets $U$, fit a random forest for each $U$, and recover Shapley effects by solving the linear regression problem defined in Section \ref{sec_algo}. Notice that we limit tree depth to $6$ when $|U| \leq 2$ to avoid overfitting---otherwise all default forest parameters are used.
We implemented SAGE following Algorithm $1$ from \citet{covert2020understanding}, and setting $m = 30$.
The original implementation of \citet{broto2020variance} in the R package \texttt{sensitivity} has an exponential complexity with $p$. Even for $p=10$, we could not have the experiments done within $24$ hours when parallelized on $16$ cores. Therefore, we do not display the results for \citet{broto2020variance}, which seem to have a high bias on toy examples.
In all procedures, the number $K$ of sampled subsets $U$ is set to $500$, and we use $500$ trees for the forest growing. Each run is repeated $30$ times to estimate the standard deviations. See Appendix \ref{appendix_A} for additional experiments supporting the choice of $K$.
For both experiments, we analytically derive the theoretical Shapley effects, and display this ground truth with red crosses in Figures \ref{fig_xp_1}---see Appendix \ref{appendix_D} for the formulas. Table \ref{table_xp_competitors} provides the sum of the absolute error of Shapley estimates for all variables with respect to the theoretical Shapley effects. This cumulative error is averaged over all repetitions to make standard deviations negligible.

\paragraph{Experiment 1: a linear case.}
In the first experiment, we consider a linear model and a correlated centered Gaussian input vector of dimension $11$. The output $Y$ follows
\begin{align*}
    Y = \beta^T\bX + \varepsilon,
\end{align*}
where $\beta \in [0,1]^{11}$, and the noise $\varepsilon$ is centered, independent, and such that $\V[\varepsilon] = 0.05 \times \V[Y]$. In Experiment 1a, two copies of $\smash{X^{(2)}}$ are appended to the data as $\smash{X^{(12)}}$ and $\smash{X^{(13)}}$, and two dummy Gaussian variables $\smash{X^{(14)}}$ and $\smash{X^{(15)}}$ are also added.  We draw a sample $\Dn$ of size $n = 3000$. 
\begin{figure}
	\begin{center}
		\includegraphics[height=6.3cm,width=9cm]{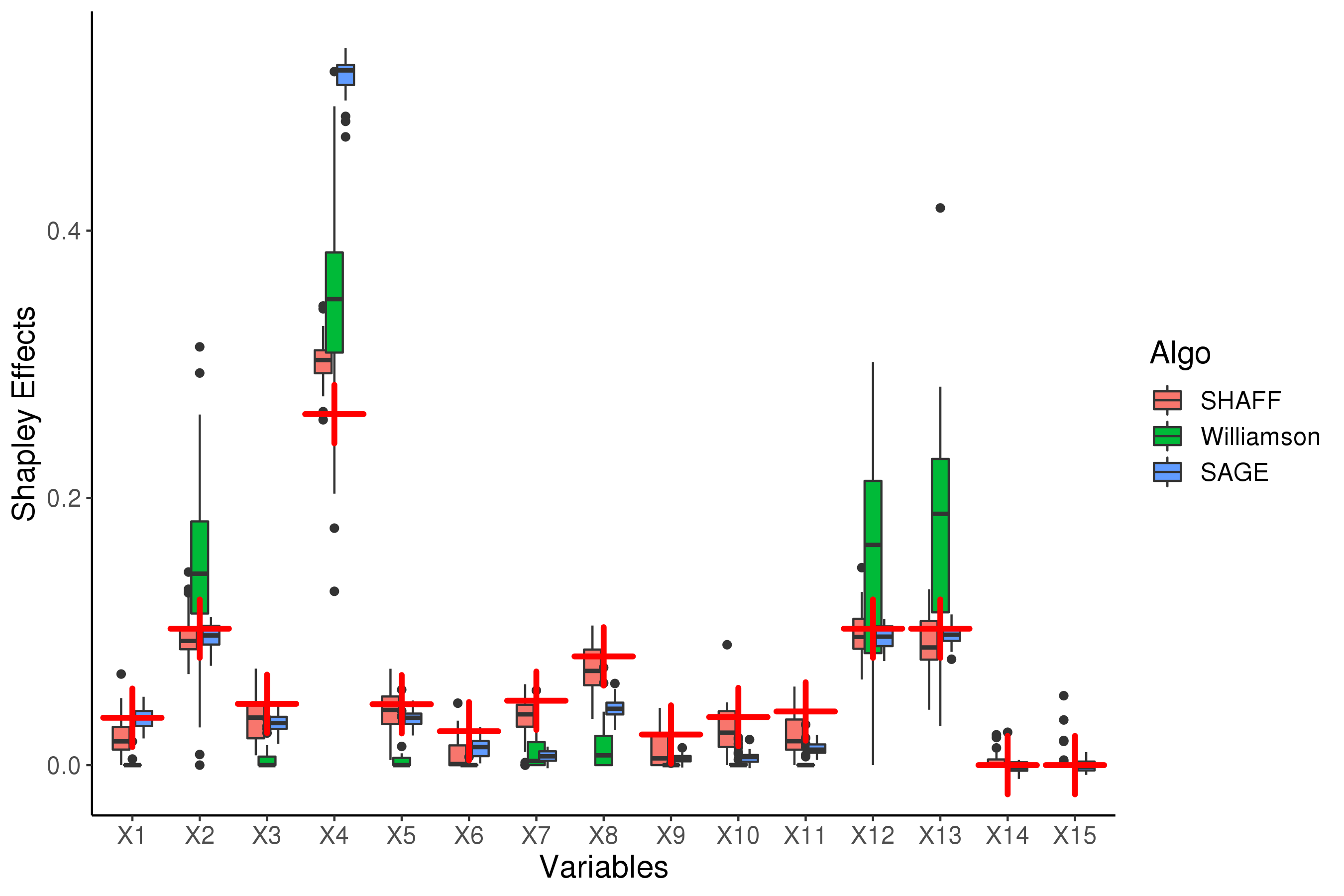}
		\caption{Shapley Effects for Experiment $1$a. (Red crosses are the theoretical Shapley effects.)}
		\label{fig_xp_1}
	\end{center}
\end{figure}
In this setting, Figure \ref{fig_xp_1} and Table \ref{table_xp_competitors} show that \textbf{SHAFF} is more accurate than its competitors. \citet[SAGE]{covert2020understanding} has a strong bias for several variables, in particular $\smash{X^{(4)}}$, $\smash{X^{(7)}}$, $\smash{X^{(8)}}$, and $\smash{X^{(10)}}$. The algorithm from \citet{williamson2020efficient} has a lower performance, and its variance is higher than for the other methods. Notice that \citet{williamson2020efficient} recommend to set $K = 2n$ ($ = 6000$ here), which is computationally more costly. Since we use $K = 500$ to compare all algorithms, this high variance is quite expected and show the improvement due to the importance sampling of our method. Besides, the computational complexity of \citet{williamson2020efficient} is $O(n^2)$ whereas \textbf{SHAFF} is quasi-linear. 
Finally, in this Experiment 1a, the random forest has a proportion of explained variance of about $86$\%, and the noise variance is $5$\%, which explains the small negative bias of many estimated values. 
Experiment 1b extends Experiment 1a to the case of higher dimension with $p = 100$, by adding $85$ noisy variables. The parameter \texttt{mtry} is set to $p$ to increase the forest accuracy in this sparse setting. Table \ref{table_xp_competitors} shows that \textbf{SHAFF} still outperforms its competitors in this context.
\begin{table}
    \setlength{\tabcolsep}{1.5pt}
	\centering
	\begin{tabular}{|c | c | c | c |}
		\hline \hline
        \small{Algorithm} & \small{Experiment} $1$a & \small{Experiment} $1$b & \small{Experiment} $2$ \\
		\hline \hline
         \small{\textbf{SHAFF}} & 0.25 & 0.80 & 0.15 \\
         \small{Williamson} & 0.64 & 1.17 & 0.24 \\
         \small{SAGE}  & 0.33 & 1.16 & 0.18 \\
		\hline \hline
	\end{tabular}
	\caption{Cumulative Absolute Error of SHAFF versus State-of-the-art Shapley Algorithms.} \label{table_xp_competitors}
\end{table}

\paragraph{Experiment 2: high-order interactions.}
In the second experiment, we consider two independent blocks of $5$ interacting variables. The input vector is Gaussian, centered, and of dimension $10$. All variables have unit variance, and all covariances are null, except
$\textrm{Cov}(X^{(1)}, X^{(2)}) = \textrm{Cov}(X^{(6)}, X^{(7)}) = 0.9$, and $\textrm{Cov}(X^{(4)}, X^{(5)}) = \textrm{Cov}(X^{(9)}, X^{(10)}) = 0.5$. 
The output $Y$ follows
\begin{align*}
 Y = & 3\sqrt{3} \times X^{(1)} X^{(2)} \mathds{1}_{X^{(3)} > 0} 
            + \sqrt{3} \times X^{(4)} X^{(5)} \mathds{1}_{X^{(3)} < 0} \\
    & + 3 \times X^{(6)} X^{(7)} \mathds{1}_{X^{(8)} > 0} + X^{(9)} X^{(10)} \mathds{1}_{X^{(8)} < 0} + \varepsilon,
\end{align*}
where the noise $\varepsilon$ is centered, independent, and such that $\V[\varepsilon] = 0.05 \times \V[Y]$. We add $5$ dummy Gaussian variables $\smash{X^{(11)}}$, $\smash{X^{(12)}}$, $\smash{X^{(13)}}$, $\smash{X^{(14)}}$, and $\smash{X^{(15)}}$, and draw a sample $\Dn$ of size $n = 10000$.
In this context of strong interactions and correlations, we observe in Table \ref{table_xp_competitors} that \textbf{SHAFF} outperforms its competitors. \textbf{SHAFF} is also the only algorithm providing the accurate variable ranking given by the theoretical Shapley effects. In particular, \textbf{SHAFF} properly identifies variable $\smash{X^{(3)}}$ as the most important one, whereas SAGE considerably overestimates the Shapley effects of variables $\smash{X^{(1)}}$ and $\smash{X^{(2)}}$---see Figure \ref{fig_xp_2} in Appendix \ref{appendix_A}.

\paragraph{SHAFF analysis.}
Table \ref{table_xp_split} displays the cumulative absolute error of Shapley algorithms, based on various combinations of variable subset sampling and conditional expectation estimates, for Experiments $1$ and $2$. The goal is to break down the improvement of \textbf{SHAFF} between the new features proposed in Section \ref{sec_algo}.
Firstly, we compare two approaches for the variable subset sampling: our paired importance sampling procedure (pIS) introduced in Subsection \ref{subsec_imp_samp}, and the paired Monte-Carlo sampling (pMC) approach of \citet{covert2020improving}.
Secondly, we compare several estimates of the conditional expectations: our projected random forest introduced in Subsection \ref{subsec_PRF}, the brute force retraining of a random forest for each subset $U$ (Forest) as in \citet{williamson2020efficient}, the marginal sampling (Marginals) used in \citet[SAGE]{covert2020understanding}, and the approach from \citet{lundberg2018consistent} specific to tree ensembles (TreeSHAP). In all cases, Shapley estimates are recovered using step $3$ defined in Subsection \ref{subsec_linreg}.
The comparisons of the first and last two lines of Table \ref{table_xp_split} clearly show the large improvement due to the importance sampling of \textbf{SHAFF}, since the cumulative error is divided by two compared to the paired Monte-Carlo sampling and using identical conditional expectation estimates. 
We also observe that the PRF algorithm is competitive with the brute force method of retraining many random forests, with a much smaller computational cost. 
Additionally, although the TreeSHAP algorithm \citep{lundberg2018consistent} is fast, it comes at the price of a much stronger bias than the other approaches.
Finally, the marginal sampling is as efficient as PRF for Experiment $1$ where the regression function is linear, but it is not the case for Experiment $2$ where variables have interactions.
\begin{table}
    \setlength{\tabcolsep}{2pt}
	\centering
	\begin{tabular}{|c | c | c |}
		\hline \hline
        Algorithm & Experiment $1$a & Experiment $2$ \\
		\hline \hline
         SHAFF & 0.25 & 0.15 \\
         pIS/Forest & 0.23 & 0.13 \\
         pIS/Marginals  & 0.26 & 0.31 \\
         pIS/TreeSHAP & 1.18 & 1.49 \\
         \hline \hline
         pMC/Projected-RF & 0.55 & 0.29 \\
         pMC/Forest & 0.56 & 0.19 \\
		\hline \hline
	\end{tabular}
	\caption{Cumulative Absolute Error of Shapley Estimates (based on various strategies for variable subset sampling and conditional expectation estimates). \label{table_xp_split}}
\end{table}

\paragraph{Experiment 3: categorical variables.}
\textbf{SHAFF} can be naturally extended to categorical variables, as it is the case for random forests. Originally, categorical variables are efficiently handled in trees by transforming them into ordered variables. Such ordering of categories is done with respect to the output mean for each category \citep{leo1984classification, friedman2001elements}, and we follow the software implementation of \texttt{ranger}.
In Experiment $3$, we consider an input vector $\bX$, made of five independent inputs, where $X^{(1)}$ and $X^{(2)}$ are two standard Gaussian variables, $X^{(3)}$ and $X^{(4)}$ are two uniform categorical variables taking $3$ possible values $a$, $b$, and $c$. We add two noisy variables $X^{(5)}$ and $X^{(6)}$, which are also uniform categorical but respectively take $10$ and $100$ different categories. Then, the regression function is given by
\begin{align*}
    m(\bX) = X^{(1)}\mathds{1}_{X^{(3)} = a} + X^{(2)}\mathds{1}_{X^{(3)} = b} + \varepsilon,
\end{align*}
where $\V[\varepsilon] = 0.05 \times \V[Y]$. Finally, a sample $\Dn$ of size $n = 2000$ is drawn to fit a random forest and Shapley algorithms. 
Figure \ref{fig_xp_cat2} and Figure \ref{fig_xp_cat1} in Appendix \ref{appendix_A} show that \textbf{SHAFF} performs well in this setting with categorical variables, as well as the brute force approach of \citet{williamson2020efficient}, as expected in this small dimensional setting. Interestingly, in the case of categorical variables, \citet{sutera2021global} and \citet{benard2021random} show that the MDI \citep{breiman2003atechnical} converges towards Shapley effects for totally randomized forests. In the case of Breiman's forests, where splits are optimized, such convergence result does not hold. In particular, $X^{(4)}$, $X^{(5)}$, and $X^{(6)}$ are not involved in the regression function and are independent of the other variables, leading to null Shapley effects. Although this is correctly estimated for $X^{(4)}$ by all algorithms, Figure \ref{fig_xp_cat2} exhibits a strong bias of the MDI for $X^{(5)}$ and $X^{(6)}$. This phenomenon is well known in the literature \citep{strobl2006bias}: the MDI is biased towards variables with a high number of categories, which is typically the case for $X^{(5)}$ and $X^{(6)}$. SAGE also suffers from this problem, whereas \textbf{SHAFF} accurately estimates a null Shapley effect for $X^{(5)}$ and $X^{(6)}$. Finally, SAGE and the MDI wrongly estimate that $X^{(3)}$ is significantly less important than $X^{(1)}$ and $X^{(2)}$.  Notice that \citet{amoukou2021accurate} provide a thorough discussion of Shapley effects for categorical variables.
\begin{figure}
	\begin{center}
		\includegraphics[height=6cm,width=8.5cm]{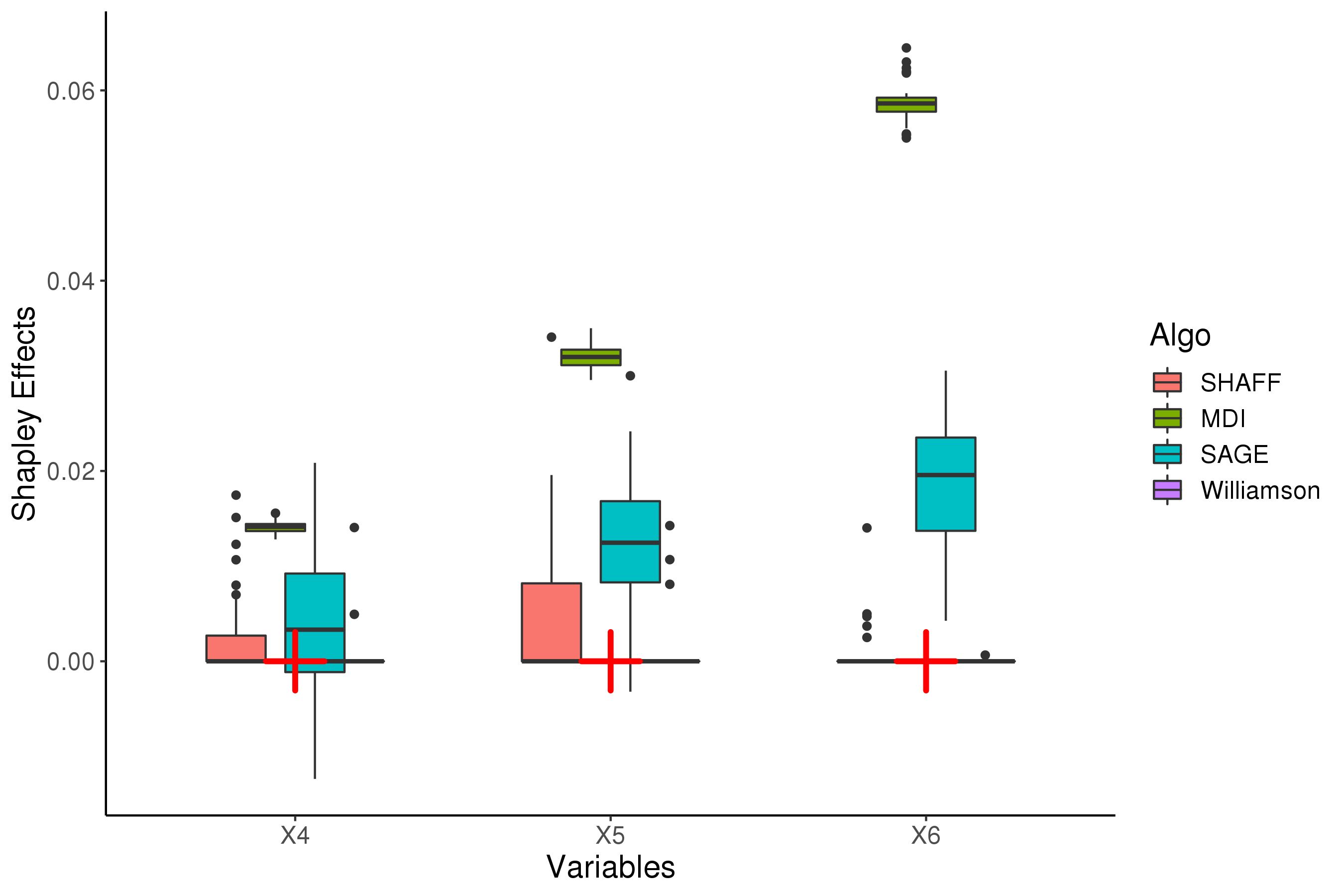}
		\caption{Shapley Effects for Experiment $3$ for $X^{(4)}$, $X^{(5)}$, and $X^{(6)}$. (Red crosses are the theoretical Shapley effects.)}
		\label{fig_xp_cat2}
	\end{center}
\end{figure}

\section{Conclusion}
We introduced \textbf{SHAFF}, \textbf{SHA}pley e\textbf{F}fects via random \textbf{F}orests, an algorithm to estimate Shapley effects based on random forests, which has an implementation in \texttt{R} and \texttt{C++}, available at \url{https://gitlab.com/drti/shaff}. The challenges in Shapley estimation are the exponential computational complexity, and the estimates of conditional expectations. \textbf{SHAFF} addresses the first point by using importance sampling to favor the subsets of influential variables, which often occur along the forest paths. For the second point, \textbf{SHAFF} uses the projected forest algorithm, a fast procedure to eliminate variables from the forest prediction mechanism. Thanks to this approach, \textbf{SHAFF} only needs to fit a random forest once, as opposed to other methods which retrain many models and are computationally costly. Importantly, we prove that \textbf{SHAFF} is consistent. To our knowledge, we propose the first Shapley algorithm which do not retrain several models and is proved to be consistent under mild assumptions. Furthermore, we conducted several experiments to show the practical performance improvements over state-of-the-art Shapley algorithms. Notice that the adaptation of \textbf{SHAFF} to SHAP values is straightforward, since the projected random forests provides predictions of the output conditional on any variable subset. In specific settings, it is obviously possible that other learning algorithms outperform random forests. Then, we can use such efficient model to generate a new large sample of simulated observations, which can then feeds \textbf{SHAFF} and improves its accuracy. Finally, the extension of the proposed algorithm to boosted tree ensembles seems a promising route for future research.

\subsubsection*{Acknowledgements}
We greatly thank the reviewers for their relevant suggestions to improve the article.

\bibliography{biblio}


\clearpage
\appendix

\thispagestyle{empty}

\onecolumn \makesupplementtitle

\FloatBarrier

\section{Additional Experiments} \label{appendix_A}

\subsection{Number of Variable Subsets $K$}
The recommended choice of $K = 500$, the number of variable subsets $U$ drawn in the first step of \textbf{SHAFF}, ensures that higher values have a small impact on \textbf{SHAFF} accuracy, while preserving a reasonable computational cost. 
For example, we sum the absolute error of \textbf{SHAFF} for all variables in Experiment 1 for increasing values of $K$, and provide the results in Table \ref{table_K} below (standard deviations are made negligible with repetitions). This shows the efficiency of the choice of $K = 500$.
\begin{table}
    \setlength{\tabcolsep}{2pt}
	\centering
	\begin{tabular}{|c | c |}
		\hline \hline
        $K$ & Cumulative Error \\
		\hline \hline
         10 & 0.67 \\
         50 & 0.40 \\
         100  & 0.30 \\
         200 & 0.29 \\
         500 & 0.25 \\
         1000 & 0.22 \\
         3000 & 0.21 \\
		\hline \hline
	\end{tabular}
	\caption{Cumulative Absolute Error of \textbf{SHAFF} in Experiment $1$ for Increasing Values of $K$. \label{table_K}}
\end{table}

\subsection{Experiment $2$}

In the context of strong interactions and correlations of Experiment $2$, we observe in Figure \ref{fig_xp_2} that all competitors have a strong bias for most variables, as opposed to \textbf{SHAFF}, which is also the only algorithm providing the accurate variable ranking given by the theoretical Shapley effects. In particular, \textbf{SHAFF} properly identifies variable $\smash{X^{(3)}}$ as the most important one, whereas SAGE considerably overestimates the Shapley effects of variables $\smash{X^{(1)}}$ and $\smash{X^{(2)}}$. 
\textbf{SHAFF} also ranks variable $\smash{X^{(8)}}$ as more important than $\smash{X^{(6)}}$ and $\smash{X^{(7)}}$, as opposed to its competitors. Besides, the proportion of explained variance of the forest is about $84$\% in this setting, which explains the negative bias observed for several estimates.
\begin{figure}
	\begin{center}
		\includegraphics[height=10cm,width=15cm]{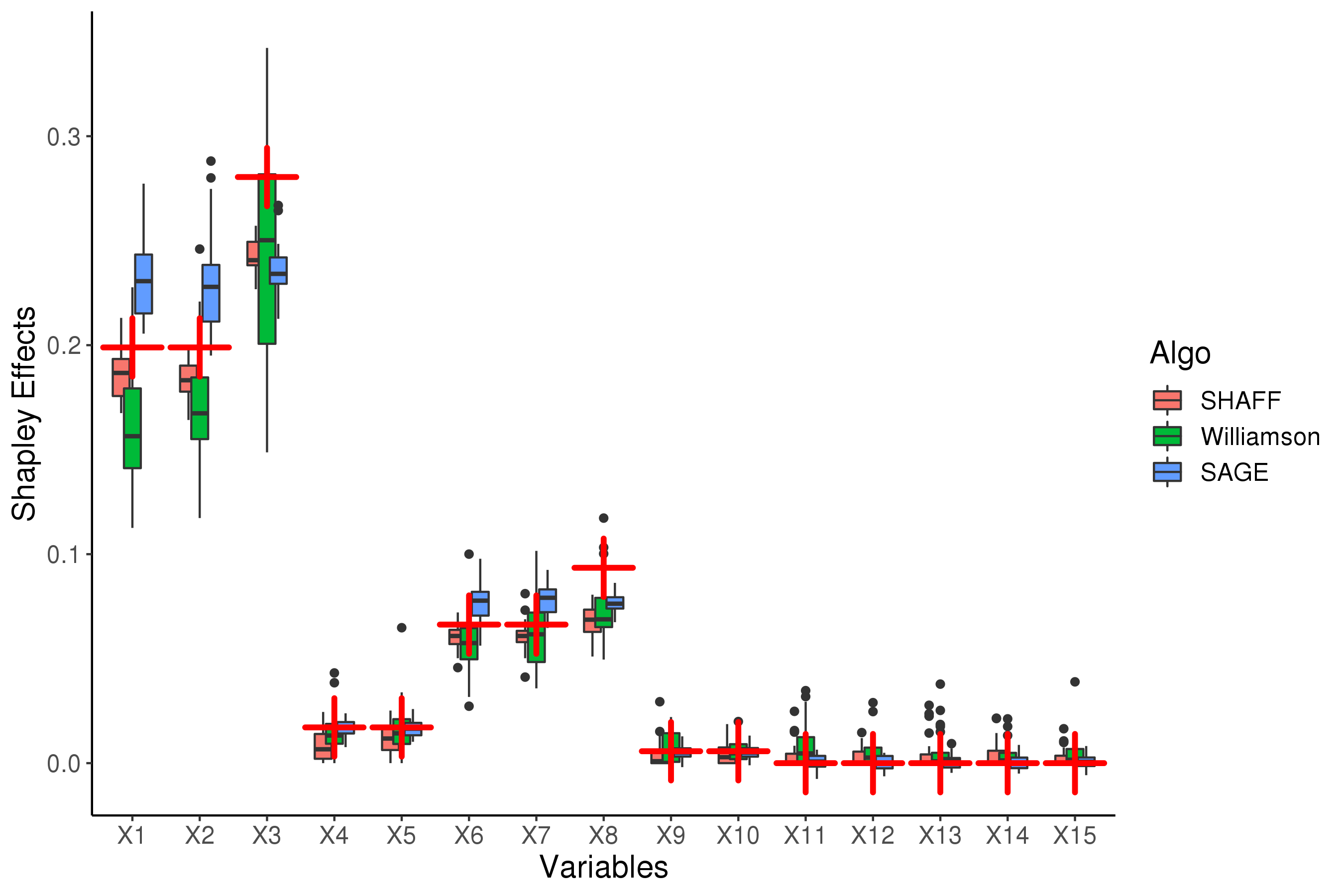}
		\caption{Shapley Effects for Experiment $2$. (Red crosses are the theoretical Shapley effects.)}
		\label{fig_xp_2}
	\end{center}
\end{figure}

\subsection{Experiment $3$}

As mentioned in the main article, SAGE and the MDI both wrongly estimate that $X^{(3)}$ is significantly less important than $X^{(1)}$ and $X^{(2)}$, as shown in Figure \ref{fig_xp_cat1}. On the other hand, \textbf{SHAFF} and SAGE are both efficient estimates of Shapley effects in this experiment.
\begin{figure}
	\begin{center}
		\includegraphics[height=8cm,width=12cm]{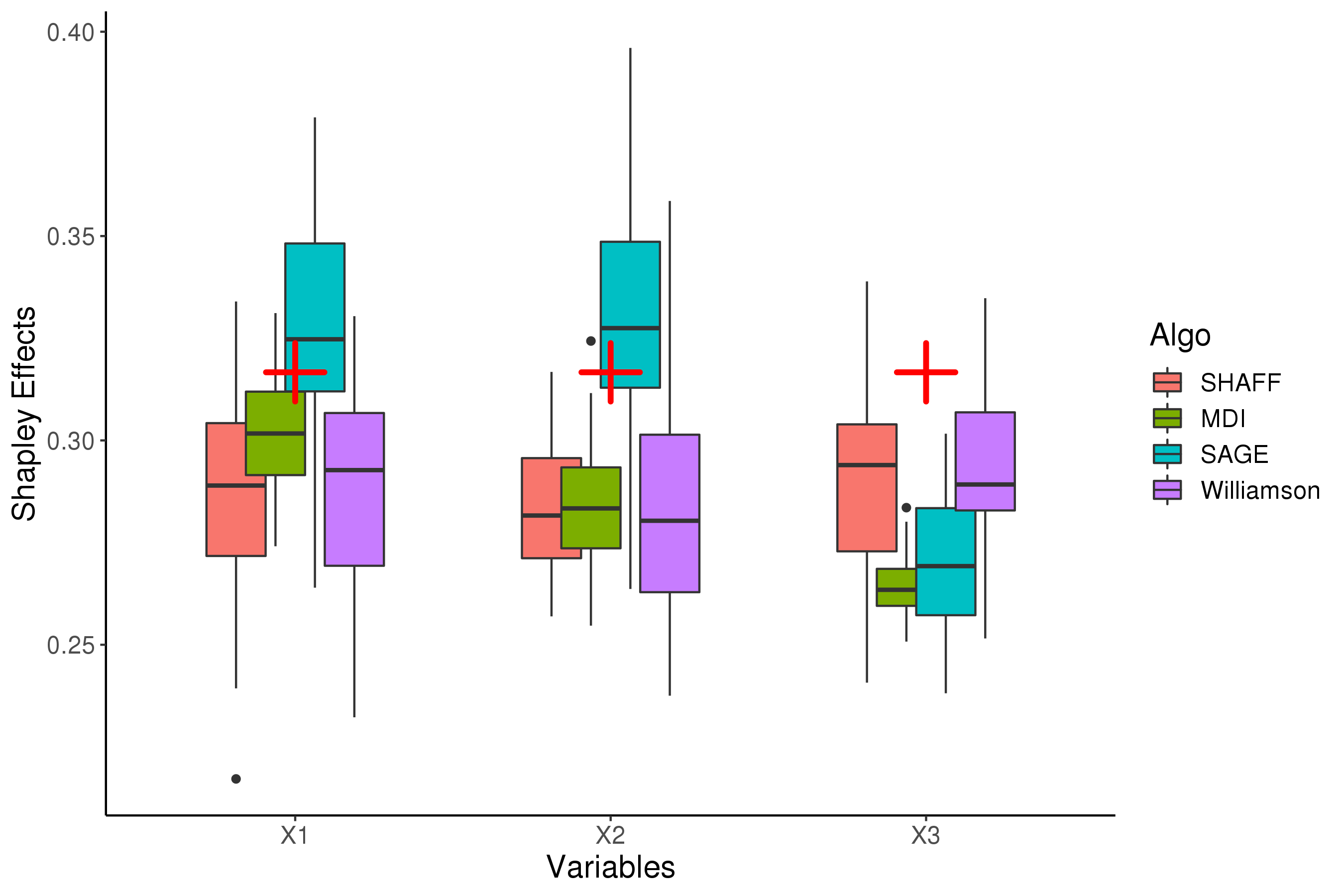}
		\caption{Shapley Effects for Experiment $3$ for $X^{(1)}$, $X^{(2)}$, and $X^{(3)}$. (Red crosses are the theoretical Shapley effects.)}
		\label{fig_xp_cat1}
	\end{center}
\end{figure}

\section{Computational Complexity} \label{appendix_B}

We provide the average computational complexity of \textbf{SHAFF}, as well as its competitors \citet{broto2020variance}, \citet{williamson2020efficient}, and \citet[SAGE]{covert2020understanding}. For these last two algorithms, random forests are used as the required black-box model. Only \textbf{SHAFF} is quasi-linear with the sample size $n$ and independent of the dimension $p$.

\subsection{SHAFF}

We derive the computational complexity of each step of \textbf{SHAFF}. Overall, the computational complexity is $O(MKn\log(n))$.

\paragraph{Importance sampling.}
In order to compute the variable subset importance, \textbf{SHAFF} counts the occurence of variable subsets $U$ in the tree paths of the forest, which has a complexity of $O(Mn)$, since each tree has about $O(n)$ nodes.
The sampling of $K$ subsets $U$ has a complexity of $O(K)$.

\paragraph{Projected random forests.}
An efficient implementation of the PRF algorithm is detailed in Algorithm \ref{algo_PRF}. For the sake of clarity, we provide a version of PRF for a single variable subset $U$ and one query point $\smash{\bX^{(U)}}$. 
Let us consider a given tree. The new observation $\smash{\bX^{(U)}}$ is dropped down the tree, eventually applying multiple splits at each level, because data points are sent on both sides of splits involving a variable outside of $U$. At the same time, the PRF computes which training observations fall in the same projected cell as $\smash{\bX^{(U)}}$, and stops going down the tree just before the size of this projected cell becomes lower than the parameter \texttt{min\_node\_size}. Such procedure has a complexity of $O(n)$ since we sequentially apply splits to reduce the number of training observations from about $n$ to \texttt{min\_node\_size} to reach the terminal projected cell. Therefore, the computational complexity to compute the PRF prediction for a given $U$ and $\smash{\bX^{(U)}}$ is $O(Mn)$. 

In \textbf{SHAFF}, the PRF is run for all subsets $U \in \bU$ and the full OOB sample for each tree.
In practice, we do not naively run Algorithm \ref{algo_PRF} for all $U$ and OOB observations, i.e., $O(Kn)$ times, since it would lead to a quadratic complexity with $n$. Instead, for a given tree, all OOB and training observations are dropped down the tree simultaneously. Even if multiple splits are applied at each tree level, we are still partitioning two samples of size $O(n)$ by sequentially applying splits: splitting one time all cells of a given partition takes $O(n)$ operations, and this has to be repeated $O(\log(n))$ times so that each cell reaches a size of \texttt{min\_node\_size}. Therefore, the global complexity of running PRF for the full OOB samples and the $K$ subsets $U$ is $O(MKn\log(n))$. 

\paragraph{Shapley effect estimates.}
The complexity to solve a least square problem with $p$ columns and $K$ rows is $O(p^3K)$. However in practice, $K$ is always fixed to default value, and when $p > K$, only at most $O(K)$ input variables are selected in the subsets $U$. For the non-selected inputs, the Shapley effect is null, and they can be removed from the least square problem, leading to a complexity of $O(K^4)$.

\subsection{Competitors}

\paragraph{\citet{broto2020variance}}
The conditional expectations are estimated for all $U \in \{1,\hdots,p\}$, which makes $2^p$ estimates. Efficient $k$-nearest neighbor algorithms have a complexity of $O(pn\log(n))$. Overall the complexity is $O(n\log(n)p2^p)$, which is exponential with respect to the dimension $p$.

\paragraph{\citet{williamson2020efficient}}
Growing $K$ random forests from scratch, one for each subset $U$, has an averaged complexity of $O(\smash{MKpn\log^2(n))}$ \citep{louppe2014understanding}. \citet{williamson2020efficient} recommend to use $K = O(n)$, which makes a global complexity of $O(\smash{Mpn^2\log^2(n))}$, and is quadratic with respect to the sample size $n$ and depends on the dimension $p$.

\paragraph{\citet[SAGE]{covert2020understanding}}
Running a prediction for random forests takes $O(M\log(n))$ operations. Since SAGE computes $np$ predictions, the global complexity is $O(Mpn\log(n))$ and depends on the dimension $p$.

\renewcommand{\labelitemi}{\textendash}
\begin{algorithm}
\caption{Projected Random Forest}
\label{algo_PRF}
\begin{algorithmic}[1]
\STATE \textbf{Inputs:} A random forest fit with $\Dn$, a variable subset $U \subset \{1,\hdots,p\}$, and a query point $\bX^{(U)}$. \\[0.5em]

\STATE for all trees in the forest: \\[0.3em]

\INDSTATE[1] \# Step 1: initialize variables
\INDSTATE[1] initialize $nodes\_level$ as a list of nodes containing only the root node;
\INDSTATE[1] initialize $nodes\_child$ as an empty list of child nodes;
\INDSTATE[1] initialize $samples$ as the list of observation indices of the full training data of the tree; \\[0.3em]

\INDSTATE[1] for all levels in the tree: \\[0.3em]

\INDSTATE[2] \# Step 2: drop $\bX^{(U)}$ to the next tree level with the relevant training observations
\INDSTATE[2] for all nodes in $nodes\_level$:
\INDSTATE[3] if the node splits on a variable in $U$: 
\INDSTATE[4] compute whether $\smash{\bX^{(U)}}$ falls in the left or right child node;
\INDSTATE[4] append the child node to $nodes\_child$;
\INDSTATE[4] set $samples\_child$ as the observations in $samples$ which satisfy the split \\ 
\INDSTATE[3] else:
\INDSTATE[4] append both the left and right children nodes to $nodes\_child$;
\INDSTATE[4] set $samples\_child = samples$;
\INDSTATE[3] if the size of $samples\_child$ is lower then $min\_node\_size$:
\INDSTATE[4] break the loop through the tree levels;
\INDSTATE[3] else:
\INDSTATE[4] set $samples = samples\_child$;
\INDSTATE[2] set $nodes\_level = nodes\_child$; \\[0.3em]

\INDSTATE[2] \# Step 3: compute prediction
\INDSTATE[2] compute the tree prediction as the average of $Y_i$ for all $i$ in $samples$; \\[0.3em]

\STATE average predictions of all trees; \\[0.3em]

\STATE return final prediction;
\end{algorithmic}
\end{algorithm}

\section{Proof of Lemmas \ref{lemma_proba}, \ref{lemma_proj}, and \ref{lemma_loss}} \label{appendix_C}

For the sake of clarity, we first recall Lemmas \ref{lemma_proba}, \ref{lemma_proj}, and \ref{lemma_loss}, and then provide their proofs.
\setcounter{lemme}{0}

\begin{lemme}
    If Assumptions (A2) and (A3) are satisfied, for all $U \subset \{1, \hdots, p\}$, we have
    \begin{align*}
        \P\big( \hat{p}_{M_n,n}(U) > 0 \big) \longrightarrow 1.    
    \end{align*}
\end{lemme}

\begin{lemme}
    If Assumptions (A1) and (A2) are satisfied, the PRF is consistent, that is, for all $M \in \mathbb{N}^{\star}$ and $U \subset \{1,\hdots,p\}$, 
    \begin{align*}
        \hat{v}_{M,n}(U) \overset{p}{\longrightarrow} \V[\E[Y|\bX^{(U)}]]/\V[Y] \overset{\rm def}{=} v^{\star}(U).
    \end{align*}
\end{lemme}

We let $Z$ be a discrete random variable taking values in the set of all subsets of $\{1, \hdots, p\}$, excluding the full and empty sets. The discrete distribution of $Z$ is given by the weights $w(U)$ (the weights are scaled to sum to 1).
\begin{lemme}
    If Assumptions (A1), (A2), and (A3) are satisfied, we have
    \begin{align*}
        \ell_{M,n}(\beta) \overset{p}{\longrightarrow} \E[(v^{\star}(Z) - \beta^T I(Z))^2] \overset{\rm def}{=} \ell^{\star}(\beta).
    \end{align*}
\end{lemme}

\begin{proof}[Proof of Lemma \ref{lemma_proba}]
    We assume that Assumptions (A2) and (A3) are satisfied, and denote by $T_{n,\ell}$ the random set of all variable subsets of $\{1,\hdots,p\}$ belonging to a path of the $\ell$-th tree. To prove the result, we derive an upper bound for $\P(\hat{p}_{M,n}(U) = 0)$. First, we write
    \begin{align*}
        \P(\hat{p}_{M,n}(U) = 0 | \Dn) = \P\big(\bigcap_{\ell=1}^{M_n} U \notin T_{n,\ell} | \Dn \big),
    \end{align*}
    and since the trees are independent conditional on $\Dn$
        \begin{align*}
        \P(\hat{p}_{M,n}(U) = 0 | \Dn) = \P(U \notin T_{n,1} | \Dn)^{M_n}.
    \end{align*}
    For $n$ large enough, there is at least one path in each tree that has at least $p$ splits. Indeed, two cases are possible to get a tree of minimum depth $p$: $n > s 2^{p-1}$, where $s$ is the minimum number of observations in a terminal leaf, or, if the maximal number of terminal leaves is reached, $t_n > 2^p$. Both are satisfied for $n$ large enough since $t_n$ is not bounded by Assumption (A2).
    Additionally, recall that the random forest algorithm is slightly modified such that \texttt{mtry} is randomly set to $1$ with a small probability $\delta$.
    Thus, if we define the random event $A_n$ as \texttt{mtry} is set to $1$ and a new variable of $U$ is selected at each node of a path of length at least $|U|$, then $A_n$ is included in $\{U \in T_{n,1}\}$. This event $A_n$ is of probability lower bounded by $(\delta/p)^p$, and thus for $n$ large enough, we have
    \begin{align*}
        \P(U \in T_{n,1} | \Dn) \geq P(A_n) \geq (\delta/p)^p,
    \end{align*}
    and then
    \begin{align*}
        \P(\hat{p}_{M,n}(U) = 0 | \Dn) \leq (1 - (\delta/p)^{p})^{M_n}.
    \end{align*}
    Finally, Assumption (A3) gives that the number of trees increases with $n$, and we obtain
    \begin{align*}
        \P\big( \hat{p}_{M,n}(U) = 0 \big) \longrightarrow 0,
    \end{align*}
    which is the desired result.
\end{proof}

\begin{proof}[Proof of Lemma \ref{lemma_proj}]
    We assume that Assumptions (A1) and (A2) are satisfied and consider $M \in \mathbb{N}^{\star}$ and $U \subset \{1,\hdots,p\}$. Recall that
    \begin{align*}
        \hat{v}_{M,n}(U) = 1 - \frac{1}{n \hat{\sigma}_Y} \sum_{i=1}^{n} \big(Y_i - m_{M,n}^{(U, OOB)}(\bX_i^{(U)}, \bTheta_M)\big)^2.
    \end{align*}
    The right hand side is expanded as follows:
    \begin{align*}
        \hat{v}_{M,n}(U) = 1 - \frac{1}{n \hat{\sigma}_Y} \sum_{i=1}^{n}& \big(m(\bX_i) + \varepsilon_i - m_{M,n}^{(U, OOB)}(\bX_i^{(U)}, \bTheta_M)\big)^2 \\
        = 1 - \frac{1}{n \hat{\sigma}_Y} \sum_{i=1}^{n}& \big(m(\bX_i) - \E[m(\bX_i)|\bX_i^{(U)}] + \varepsilon_i \\ & - [m_{M,n}^{(U, OOB)}(\bX_i^{(U)}, \bTheta_M) - \E[m(\bX_i)|\bX_i^{(U)}]] \big)^2. \\
    \end{align*}
    Therefore,
    \begin{align} \label{eq_decomposition}
        \hat{v}_{M,n}(U) = 1 - &\frac{1}{n \hat{\sigma}_Y} \sum_{i=1}^{n} (m(\bX_i) - \E[m(\bX_i)|\bX_i^{(U)}])^2 \nonumber \\ \nonumber &+ \varepsilon_i^2 + 2 \varepsilon_i \times (m(\bX_i) - \E[m(\bX_i)|\bX_i^{(U)}]) \nonumber \\ 
        & - 2 \varepsilon_i \times \big(m_{M,n}^{(U, OOB)}(\bX_i^{(U)}, \bTheta_M) - \E[m(\bX_i)|\bX_i^{(U)}]\big) \nonumber \\ 
        & - 2 (m(\bX_i) - \E[m(\bX_i)|\bX_i^{(U)}]) \times \big(m_{M,n}^{(U, OOB)}(\bX_i^{(U)}, \bTheta_M) - \E[m(\bX_i)|\bX_i^{(U)}]\big) \nonumber \\ 
        & + \big(m_{M,n}^{(U, OOB)}(\bX_i^{(U)}, \bTheta_M) - \E[m(\bX_i)|\bX_i^{(U)}]\big)^2.
    \end{align}
    Now, using the law of large numbers, we obtain
    \begin{align*}
        \frac{1}{n} \sum_{i=1}^{n} (m(\bX_i) -& \E[m(\bX_i)|\bX_i^{(U)}])^2 + \varepsilon_i^2 \\[-1em] & + 2 \varepsilon_i \times (m(\bX_i) - \E[m(\bX_i)|\bX_i^{(U)}])
        \overset{p}{\longrightarrow} \E[\V[m(\bX) | \bX^{(U)}]] + \V[\varepsilon],
    \end{align*}
    and also $\hat{\sigma}_Y \overset{p}{\longrightarrow} \V[Y]$.
    Combining these two limits, we have
    \begin{align*}
        1 - \frac{1}{n \hat{\sigma}_Y} \sum_{i=1}^{n} (&m(\bX_i) - \E[m(\bX_i)|\bX_i^{(U)}])^2 + \varepsilon_i^2 \\[-1em] & + 2 \varepsilon_i \times (m(\bX_i) - \E[m(\bX_i)|\bX_i^{(U)}])
        \overset{p}{\longrightarrow}  1 - (\E[\V[m(\bX) | \bX^{(U)}]] + \V[\varepsilon])/\V[Y].
    \end{align*}
    Rewriting this limit using the law of total variance, we are led to
    \begin{align*}
        1 - (\E&[\V[m(\bX) | \bX^{(U)}]] + \V[\varepsilon])/\V[Y] \\
        & = (\V[Y] - \E[\V[m(\bX) | \bX^{(U)}]] + \V[\varepsilon])/\V[Y] \\ 
        & = (\V[m(\bX)] + \V[\varepsilon] - \E[\V[m(\bX) | \bX^{(U)}]] - \V[\varepsilon])/\V[Y] \\
        & = \V[\E[m(\bX) | \bX^{(U)}]]/\V[Y] \\
        & = \V[\E[Y | \bX^{(U)}]]/\V[Y] \\
        & = v^{\star}(U).
    \end{align*}
    
    Overall, the result of the lemma holds if the last three terms of the decomposition (\ref{eq_decomposition}) converge towards $0$ in probability. This is clearly true if the OOB PRF estimate is $\mathbb{L}^2$-consistent, that is for $i \in \{1,\hdots,n\}$,
    \begin{align*}
        \E\big[ \big(m_{M,n}^{(U, OOB)}(\bX_i^{(U)}, \bTheta_M) - \E[m(\bX_i)|\bX_i^{(U)}]\big)^2 \big] 
        \longrightarrow 0.
    \end{align*}
    
    According to Lemma $2$ from \citet{benard2021mda}, the $\mathbb{L}^2$-convergence of the OOB forest estimate follows from the convergence of the standard forest estimate. Therefore, we only need to show the $\mathbb{L}^2$-convergence of the PRF estimate to get the final result. To do so, we adapt the proof of Theorem $1$ from \citet{scornet2015consistency}, which shows the convergence of Breiman's forests for additive models.
 
    The proof only differs for the approximation error. Indeed, we need to show that the variation of the regression function vanishes in a cell of the empirical PRF. \citet{scornet2015consistency} show that this is always true in the original forest for additive models. Here, the result is valid for all regression functions, using the fact that the random forest is slightly modified: splits cannot be too close from the edges of cells (at least a fraction of $\gamma$ observations in children nodes), and $mtry$ is set to $1$ at each node with a small probability $\delta$.
    Under these small modifications, Lemma $2$ from \citet{meinshausen2006quantile} gives that the diameter of each cell of the original forest vanishes, i.e, 
    \begin{align*}
        \lim \limits_{n \to \infty} \textrm{diam}(A_n(\bX, \Theta)) = 0,
    \end{align*}
    where $A_n(\bX, \Theta)$ is the cell of the forest where the new query point $\bX$ falls, and the diameter of a cell $A$ is the length of the longest line fitting in $A$, formally
    \begin{align*}
        \textrm{diam}(A) = \sup_{\bx,\bx' \in A} ||\bx - \bx'||_2.
    \end{align*}
    By definition of the PRF algorithm, the projected cell where $\bX^{(U)}$ falls is included in $A_n(\bX, \Theta)$, and therefore the diameter of the projected cell also vanishes as $n$ increases. Additionally, the regression function $m$ is continuous by Assumption (A1), and consequently the approximation error converges to $0$.
    Finally, the PRF estimate is $\mathbb{L}^2$-consistent, and we deduce the final result, 
    \begin{align*}
        \hat{v}_{M,n}(U) \overset{p}{\longrightarrow} v^{\star}(U).
    \end{align*}
\end{proof}

\begin{proof}[Proof of Lemma \ref{lemma_loss}]

    The loss function $\ell_{M,n}$ contains three sources of randomness: the data $\Dn$, the forest randomization $\Theta$, and the importance sampling of the subsets $U$. The discrete distribution used to sample the subsets $U$ is built using the occurrence frequency in the forest $\hat{p}_{M,n}(U)$, which depends on $\Dn$ and $\Theta$. This subtle relation between the data, the forest, and the importance sampling prevent a straightforward proof for this lemma.
    We reshape the loss function and use the law of total variance to handle separately the multiple sources of randomness. We assume that Assumptions (A1), (A2), and (A3) are satisfied.
    
    First, we have
    \begin{align*}
        \ell_{M,n}(\beta) &= \frac{1}{K_n} \sum_{U \in \bU} \frac{w(U)}{\hat{p}_{M,n}(U)} (\hat{v}_{M,n}(U) - \beta^T I(U))^2\\
        &= \sum_{U \subset \{1,\hdots,p\}} \frac{w(U) N_n(U)}{K_n \hat{p}_{M,n}(U)} \mathds{1}_{\hat{p}_{M,n}(U) > 0} (\hat{v}_{M,n}(U) - \beta^T I(U))^2,
    \end{align*}
    where $N_n(U)$ is the number of times where $U$ is drawn in $\bU$ (with the convention $0/0 = 0$). Since the sum is finite, it is enough to study the convergence of the terms one by one. Let us consider a given variable subset $U$.
    First, we define
    \begin{align*}
        \Delta_{n,K_n} = \frac{N_n(U) \mathds{1}_{\hat{p}_{M,n}(U) > 0}}{K_n \hat{p}_{M,n}(U)}.
    \end{align*}
    Next, we derive the limit of $\V[\Delta_{n,K_n}]$ using the law of total variance. We have
    \begin{align*}
        \V[\Delta_{n,K_n}] = \E[\V[\Delta_{n,K_n}|\Dn,\Theta]] + \V[\E[\Delta_{n,K_n}|\Dn,\Theta]].
    \end{align*}
    On one hand, since $K_n$ is a constant and $\hat{p}_{M,n}(U)$ only depends on $\Dn$ and $\Theta$, we have
    \begin{align*}
        \V[\Delta_{n,K_n}|\Dn,\Theta] = \V\big[\frac{N_n(U) \mathds{1}_{\hat{p}_{M,n}(U) > 0}}{K_n \hat{p}_{M,n}(U)}|\Dn,\Theta \big]
        = \Big(\frac{\mathds{1}_{\hat{p}_{M,n}(U) > 0}}{K_n \hat{p}_{M,n}(U)}\Big)^2 \V\big[N_n(U) |\Dn,\Theta \big].
    \end{align*}
    By definition, $N_n(U) = \sum_{k = 1}^{K_n} \mathds{1}_{U_k = U}$, where $U_1,\hdots,U_{K_n}$ are the variable subsets drawn at each iteration of the importance sampling. 
    Since $U_1,\hdots,U_{K_n}$ are independent conditional on $\Dn$ and $\Theta$, and $U$ is drawn with probability $\hat{p}_{M,n}(U)$,
    \begin{align*}
        \V\big[N_n(U) |\Dn,\Theta \big] = K_n \V[\mathds{1}_{U_1 = U} | \Dn,\Theta] = K_n \hat{p}_{M,n}(U) [1 - \hat{p}_{M,n}(U)],
    \end{align*}
    and finally
    \begin{align*}
        \E[\V[\Delta_{n,K_n}|\Dn,\Theta]] = \frac{1}{K_n} \E\big[\frac{1 - \hat{p}_{M,n}(U)}{\hat{p}_{M,n}(U)} \mathds{1}_{\hat{p}_{M,n}(U) > 0}\big].
    \end{align*}
    Therefore,
    \begin{align*}
        \E[\V[\Delta_{n,K_n}|\Dn,\Theta]] \leq \frac{1}{K_n} \E\big[\frac{\mathds{1}_{\hat{p}_{M,n}(U) > 0}}{\hat{p}_{M,n}(U)}\big].
    \end{align*}
    The number of paths in the forest is upper bounded by $n \times M_n$, and therefore if $\hat{p}_{M,n}(U)$ is not null, it is lower bounded by $1/(n.M_n)$. Thus
    \begin{align*}
        \E[\V[\Delta_{n,K_n}|\Dn,\Theta]] \leq \frac{n.M_n}{K_n},
    \end{align*}
    which converges to $0$ by Assumption (A3).
    
    On the other hand,
    \begin{align*}
        \E[\Delta_{n,K_n}|\Dn,\Theta] = 
        \frac{\mathds{1}_{\hat{p}_{M,n}(U) > 0}}{K_n \hat{p}_{M,n}(U)} \E[N_n(U)|\Dn,\Theta] = \mathds{1}_{\hat{p}_{M,n}(U) > 0},
    \end{align*}
    and then 
    \begin{align*}
        \V[\E[\Delta_{n,K_n}|\Dn,\Theta]] =& \P(\hat{p}_{M,n}(U) > 0)[1 - \P(\hat{p}_{M,n}(U) > 0)] \\
        =& \P(\hat{p}_{M,n}(U) > 0)\P(\hat{p}_{M,n}(U) = 0).
    \end{align*}
   Lemma \ref{lemma_proba} gives that $\P\big( \hat{p}_{M,n}(U) = 0 \big) \longrightarrow 0$, which implies the convergence of $\V[\E[\Delta_{n,K_n}|\Dn,\Theta]]$ towards $0$. 
    
    Overall, the law of total variance gives that
    \begin{align*}
        \V[\Delta_{n,K_n}] \longrightarrow 0.
    \end{align*}
    Since $\E[\Delta_{n,K_n}] = \P\big( \hat{p}_{M,n}(U) > 0 \big) \longrightarrow 1$ and $\mathbb{L}^2$-convergence implies convergence in probability, we have
    \begin{align*}
        \Delta_{n,K_n} \overset{p}{\longrightarrow} 1.
    \end{align*}

    Next, using Lemma \ref{lemma_proj}, we obtain
    \begin{align*}
        \frac{w(U) N_n(U)}{K_n \hat{p}_{M,n}(U)} \mathds{1}_{\hat{p}_{M,n}(U) > 0} (\hat{v}_{M,n}(U) - \beta^T I(U))^2 \overset{p}{\longrightarrow} 
        w(U) (v^{\star}(U) - \beta^T I(U))^2.
    \end{align*}
    If $Z$ is a discrete random variable taking values in the set of all subsets of $\{1,\hdots,p\}$, excluding the full and empty sets, and distributed with the scaled weights $w(U)$, we finally have
    \begin{align*}
        \ell_{M,n}(\beta) \overset{p}{\longrightarrow} \E[(v^{\star}(Z) - \beta^T I(Z))^2].
    \end{align*}
\end{proof}

\section{Formulas of Theoretical Shapley Effects for Experiments} \label{appendix_D}

\paragraph{Experiment 1.}
For a linear model with a Gaussian input vector of dimension $p$, the theoretical Shapley effects are given by Theorem $2$ in \citep{owen2017shapley} as
\begin{align*}
    Sh^{\star}(X^{(j)}) = \frac{1}{p} \sum_{U \subset \{1,\hdots,p\} \setminus j} {p - 1 \choose |U|}^{-1} \frac{\textrm{Cov}[X^{(j)}, \bX^{(-U) T} \beta^{(-U)} | \bX^{(U)}]^2}{\V[X^{(j)}|\bX^{(U)}]} \Big(1 - \frac{\sigma_{\varepsilon}^2}{\V[Y]}\Big),
\end{align*}
where the conditional covariances and variances can be easily computed using standard formulas for Gaussian vectors, and $\sigma_{\varepsilon}^2$ is the noise variance.

In Experiment 1, several copies of a given input $X^{(k)}$ are added to the data. We denote by $r$ the number of redundant variables. We easily deduce the updated value $Sh'^{\star}(X^{(j)})$ from the original Shapley effects $Sh^{\star}(X^{(j)})$ for all variables. Then, we have
\begin{align*}
    Sh'^{\star}(X^{(k)}) = \frac{1}{p + r} \sum_{U \subset \{1,\hdots,p\} \setminus k} {p + r - 1 \choose |U|}^{-1} \frac{\textrm{Cov}[X^{(k)}, \bX^{(-U) T} \beta^{(-U)} | \bX^{(U)}]^2}{\V[X^{(k)}|\bX^{(U)}]} \Big(1 - \frac{\sigma_{\varepsilon}^2}{\V[Y]}\Big).
\end{align*}

If $j \in \{1,\hdots,p\} \setminus k$, we have
\begin{align*}
    Sh'^{\star}(X^{(j)}) = \frac{1}{p + r} \sum_{\small{\begin{array}{c}
        U \subset \{1,\hdots,p\} \setminus j \\
        \textrm{s.t. } k \notin U\end{array}}}& {p + r - 1 \choose |U|}^{-1} \frac{\textrm{Cov}[X^{(j)}, \bX^{(-U) T} \beta^{(-U)} | \bX^{(U)}]^2}{\V[X^{(j)}|\bX^{(U)}]} \Big(1 - \frac{\sigma_{\varepsilon}^2}{\V[Y]}\Big) \\
        + \frac{1}{p + r} \sum_{\small{\begin{array}{c}
        U \subset \{1,\hdots,p\} \setminus j \\
        \textrm{s.t. } k \in U\end{array}}} 
        & \Big[ \sum_{\ell=0}^{r} {r \choose \ell}{p + r - 1 \choose |U| + \ell}^{-1}
              + \sum_{\ell=1}^{r} {r \choose \ell}{p + r - 1 \choose |U| + \ell - 1}^{-1} \Big] \\
        \\[-3em] & \times \frac{\textrm{Cov}[X^{(j)}, \bX^{(-U) T} \beta^{(-U)} | \bX^{(U)}]^2}{\V[X^{(j)}|\bX^{(U)}]} \Big(1 - \frac{\sigma_{\varepsilon}^2}{\V[Y]}\Big).
\end{align*}

Finally, for $j \in \{p+1, \hdots, p+r\}$, clearly
\begin{align*}
    Sh'^{\star}(X^{(j)}) = Sh'^{\star}(X^{(k)}),
\end{align*}
and dummy variables have a null Shapley effect.

\paragraph{Experiment 2.}
Recall that in the second experiment, we consider two independent blocks of $5$ interacting variables. The input vector is Gaussian, centered, and of dimension $10$. All variables have unit variance, and all covariances are null, except
$\textrm{Cov}(X^{(1)}, X^{(2)}) = \textrm{Cov}(X^{(6)}, X^{(7)}) = \rho_1$, and $\textrm{Cov}(X^{(4)}, X^{(5)}) = \textrm{Cov}(X^{(9)}, X^{(10)}) = \rho_2$. 
The output $Y$ is defined as a specific case of
\begin{align*}
 Y = a \sqrt{\alpha}& \times X^{(1)} X^{(2)} \mathds{1}_{X^{(3)} > 0} 
            + b \sqrt{\alpha} \times X^{(4)} X^{(5)} \mathds{1}_{X^{(3)} < 0} \\
    & + c \sqrt{\beta} \times X^{(6)} X^{(7)} \mathds{1}_{X^{(8)} > 0} + d \sqrt{\beta} X^{(9)} X^{(10)} \mathds{1}_{X^{(8)} < 0} + \varepsilon.
\end{align*}

The Shapley effects of the input variables are given by
\begin{align*}
    Sh^{\star}(X^{(1)}) = Sh^{\star}(X^{(2)}) = \frac{\alpha}{\alpha V_1 + \beta V_2 + \sigma^2_{\varepsilon}} \Big( \frac{(a\rho_1)^2}{8} + \frac{5}{24}a^2 \Big),
\end{align*}
\begin{align*}
    Sh^{\star}(X^{(4)}) = Sh^{\star}(X^{(5)}) = \frac{\alpha}{\alpha V_1 + \beta V_2 + \sigma^2_{\varepsilon}} \Big( \frac{(b\rho_2)^2}{8} + \frac{5}{24}b^2 \Big),
\end{align*}
\begin{align*}
    Sh^{\star}(X^{(3)}) = \frac{\alpha}{\alpha V_1 + \beta V_2 + \sigma^2_{\varepsilon}} \Big( \frac{(a\rho_1 - b\rho_2)^2}{4} 
    + \frac{(a\rho_1)^2}{4} + \frac{(b\rho_2)^2}{4} + \frac{a^2}{12} + \frac{b^2}{12} \Big),
\end{align*}
where 
\begin{align*}
   V_1 = \Big( \frac{(a\rho_1 - b\rho_2)^2}{4} 
    + \frac{(a\rho_1)^2}{2} + \frac{(b\rho_2)^2}{2} + \frac{a^2}{2} + \frac{b^2}{2} \Big),
\end{align*}
and 
\begin{align*}
   V_2 = \Big( \frac{(c\rho_1 - d\rho_2)^2}{4} 
    + \frac{(c\rho_1)^2}{2} + \frac{(d\rho_2)^2}{2} + \frac{c^2}{2} + \frac{d^2}{2} \Big).
\end{align*}
Symmetrically, we have
\begin{align*}
    Sh^{\star}(X^{(6)}) = Sh^{\star}(X^{(7)}) = \frac{\beta}{\alpha V_1 + \beta V_2 + \sigma^2_{\varepsilon}} \Big( \frac{(c\rho_1)^2}{8} + \frac{5}{24}c^2 \Big),
\end{align*}
\begin{align*}
    Sh^{\star}(X^{(9)}) = Sh^{\star}(X^{(10)}) = \frac{\beta}{\alpha V_1 + \beta V_2 + \sigma^2_{\varepsilon}} \Big( \frac{(d\rho_2)^2}{8} + \frac{5}{24}d^2 \Big),
\end{align*}
\begin{align*}
    Sh^{\star}(X^{(8)}) = \frac{\beta}{\alpha V_1 + \beta V_2 + \sigma^2_{\varepsilon}} \Big( \frac{(c\rho_1 - d\rho_2)^2}{4} 
    + \frac{(c\rho_1)^2}{4} + \frac{(d\rho_2)^2}{4} + \frac{c^2}{12} + \frac{d^2}{12} \Big).
\end{align*}
Clearly, $ Sh^{\star}(X^{(11)}) = Sh^{\star}(X^{(12)}) = Sh^{\star}(X^{(13)}) = Sh^{\star}(X^{(14)}) =Sh^{\star}(X^{(15)}) = 0$.

\end{document}